\documentclass[conference]{IEEEtran}

\usepackage{times}

\let\labelindent\relax

\usepackage[numbers]{natbib}
\usepackage{multicol}
\usepackage[bookmarks=true]{hyperref}

	\usepackage{amsmath}
	\usepackage{amsthm}
	\usepackage{amsfonts}
	\usepackage{amssymb}
	\usepackage{ragged2e}
	\usepackage{graphicx}
	\usepackage{cancel}
	\usepackage{mathtools}
	\usepackage{tabularx}
	\usepackage{arydshln}
	\usepackage{tensor}
	\usepackage{array}
	\usepackage[dvipsnames]{xcolor}
	\usepackage[boxed]{algorithm}
	\usepackage[noend]{algpseudocode}
	\usepackage{listings}
	\usepackage{textcomp}
	\usepackage{mathrsfs}
	\usepackage{bbm}
	\usepackage{tikz}
	\usepackage{tikz-cd}
	\usepackage{enumitem}
	\usepackage{arydshln}
	\usepackage{relsize}
	\usepackage{multirow}
	\usepackage{scalerel}
	\usepackage{upgreek}
	\usepackage{ifthen}
	\usepackage{yhmath}
	\usepackage{blkarray}
	\usepackage{dashrule}
	\usepackage{subcaption}
	\usepackage{overpic}

		\newcommand{\inv}{^{-1}}
		\newcommand{\abs}[1]{\left|#1\right|}

		\newcommand{\of}{\circ}

		\newcommand{\restr}[1]{\left.#1\right|}

		\newcommand{\card}[1]{\left|#1\right|}
		\newcommand{\set}[1]{\left\{#1\right\}}
		\newcommand{\setmid}{\;\middle|\;}

		\newcommand{\reals}{\mathbb{R}}
		\newcommand{\R}{\reals}

		\newcommand{\norm}[1]{\abs{\abs{#1}}}
		
		\newcommand{\iprod}[1]{\left<#1\right>}

		\DeclareMathOperator{\proj}{proj}

		\newcommand{\closure}[1]{\overline{#1}}

		\DeclareMathOperator{\argmin}{argmin}

		\DeclareMathOperator{\piecewise}{pw}
		
		\newcommand{\st}{s.t.}

		\newcommand{\paren}[1]{\left(#1\right)}

		\newcommand{\ptxt}[1]{\textrm{\textnormal{#1}}}
		
		\newcommand{\mc}[1]{\mathcal{#1}}
		
		\newcommand{\mb}[1]{\mathbb{#1}}
		\newtheorem{theorem}{Theorem}
		\newtheorem{assumption}{Assumption}
		
		\newtheorem{lemma}{Lemma}

	\usepackage{letltxmacro}
	\LetLtxMacro\orgvdots\vdots
	\LetLtxMacro\orgddots\ddots

	\makeatletter
	\DeclareRobustCommand\vdots{%
		\mathpalette\@vdots{}%
	}
	\newcommand*{\@vdots}[2]{%
		\sbox0{$#1\cdotp\cdotp\cdotp\m@th$}%
		\sbox2{$#1.\m@th$}%
		\vbox{%
			\dimen@=\wd0 %
			\advance\dimen@ -3\ht2 %
			\kern.5\dimen@
			\dimen@=\wd2 %
			\advance\dimen@ -\ht2 %
			\dimen2=\wd0 %
			\advance\dimen2 -\dimen@
			\vbox to \dimen2{%
				\offinterlineskip
				\copy2 \vfill\copy2 \vfill\copy2 %
			}%
		}%
	}
	\DeclareRobustCommand\ddots{%
		\mathinner{%
			\mathpalette\@ddots{}%
			\mkern\thinmuskip
		}%
	}
	\newcommand*{\@ddots}[2]{%
		\sbox0{$#1\cdotp\cdotp\cdotp\m@th$}%
		\sbox2{$#1.\m@th$}%
		\vbox{%
			\dimen@=\wd0 %
			\advance\dimen@ -3\ht2 %
			\kern.5\dimen@
			\dimen@=\wd2 %
			\advance\dimen@ -\ht2 %
			\dimen2=\wd0 %
			\advance\dimen2 -\dimen@
			\vbox to \dimen2{%
				\offinterlineskip
				\hbox{$#1\mathpunct{.}\m@th$}%
				\vfill
				\hbox{$#1\mathpunct{\kern\wd2}\mathpunct{.}\m@th$}%
				\vfill
				\hbox{$#1\mathpunct{\kern\wd2}\mathpunct{\kern\wd2}\mathpunct{.}\m@th$}%
			}%
		}%
	}
	\makeatother

	\tikzset{
	  symbol/.style={
		draw=none,
		every to/.append style={
		  edge node={node [sloped, allow upside down, auto=false]{$#1$}}}
	  }
	}

\usepackage{cleveref}
\Crefname{figure}{Fig.}{Figs.}
\Crefname{equation}{Eq.}{Eqs.}
\Crefname{lemma}{Lemma}{Lemmata}
\Crefname{proposition}{Proposition}{Propositions}
\Crefname{assumption}{Assumption}{Assumptions}
\Crefname{theorem}{Theorem}{Theorems}
\Crefname{section}{Section}{Sections}
\Crefname{subsection}{Subsection}{Subsections}
\Crefname{appendix}{Appendix}{Appendices}

\newcommand{\dist}{\mathsf{d}}
\newcommand{\rmd}{\mathrm{d}}
\newcommand{\Cpw}{\mc{C}^{1}_{\piecewise}}
\newcommand{\CpwM}{\mc{C}^{1}_{\piecewise}([0,1],\closure{\mc{M}})}
\newcommand{\bz}{B\'ezier}
\DeclareMathOperator{\SO}{SO}
\DeclareMathOperator{\SE}{SE}
\hypersetup{
	pdfinfo={
		Title={Non-Euclidean Motion Planning with Graphs of Geodesically-Convex Sets},
		Author={Thomas Cohn, Mark Petersen, Max Simchowitz, and Russ Tedrake},
		Subject={Robot Planning},
		Keywords={Motion Planning;Optimization;Riemannian Geometry}
	}
}

\begin{document}

\title{Non-Euclidean Motion Planning with Graphs of Geodesically-Convex Sets}

\author{
	\authorblockN{
		Thomas Cohn\authorrefmark{1},
		Mark Petersen\authorrefmark{2},
		Max Simchowitz\authorrefmark{1}, and
		Russ Tedrake\authorrefmark{1}
	}
	\authorblockA{
		\authorrefmark{1}Computer Science and Artificial Intelligence Laboratory (CSAIL)\\
		Massachusetts Institute of Technology,
		Cambridge, Massachusetts 02139--4309\\ Email: \texttt{\{tcohn,msimchow,russt\}@mit.edu}
	}
	\authorblockA{
		\authorrefmark{2}School of Engineering and Applied Sciences (SEAS)\\
		Harvard University,
		Cambridge, Massachusetts 02138--2933\\ Email: \texttt{markpetersen@g.harvard.edu}
	}
}

\maketitle

\begin{abstract}
Computing optimal, collision-free trajectories for high-dimensional systems is a challenging problem. Sampling-based planners struggle with the dimensionality, whereas trajectory optimizers may get stuck in local minima due to inherent nonconvexities in the optimization landscape. The use of mixed-integer programming to encapsulate these nonconvexities and find globally optimal trajectories has recently shown great promise, thanks in part to tight convex relaxations and efficient approximation strategies that greatly reduce runtimes. These approaches were previously limited to Euclidean configuration spaces, precluding their use with mobile bases or continuous revolute joints. In this paper, we handle such scenarios by modeling configuration spaces as Riemannian manifolds, and we describe a reduction procedure for the zero-curvature case to a mixed-integer convex optimization problem. We demonstrate our results on various robot platforms, including producing efficient collision-free trajectories for a PR2 bimanual mobile manipulator.
\end{abstract}

\IEEEpeerreviewmaketitle

\section{Introduction}
\label{sec:introduction}

Planning the motion of robots through environments with obstacles is a long-standing and ever-present problem in robotics.
In this paper, we aim to find the shortest path between a start and goal configuration
 with guaranteed collision avoidance.
We are particularly motivated by planning for bimanual mobile manipulators, such as the PR2 (Willow Garage).
Such robots are well-suited for a variety of tasks in human environments but present various challenges for existing motion planning algorithms.

Most popular approaches for this task fall into two categories: sampling-based planners and trajectory optimizers.
Trajectory optimizers formulate the motion planning problem as an optimization problem.
This problem is inherently nonconvex when there are obstacles in the scene, so trajectory optimizers frequently get stuck in local minima.
In that case, they may output a path that is longer than the global optimum or even fail to produce a valid path at all.

On the other hand, sampling-based planners can avoid getting stuck in local minima, but the path may be locally suboptimal, resulting in jerky and uneven motion.
Sampling-based planners may also suffer from the so-called ``Curse of Dimensionality''.
Because they rely on covering the configuration space with discrete samples, in the worst case, the number of samples required may increase exponentially with the dimension.
The PR2 has two 7-DoF arms and a mobile base, and sampling-based planners struggle with the instances we study here.

Recently, \citet{marcucci2021shortest,marcucci2022motion} described a new type of motion planning, based on a decomposition of the collision-free subset of configuration space (C-Free) into convex sets.
They leverage a new optimization framework, a \emph{Graph of Convex Sets} (GCS), where each vertex is associated with a convex set and each edge is associated with a convex function.
Motion planning becomes a shortest-path problem in this space.
This GCS-Planner has been successfully applied to challenging, high-dimensional problems, including bimanual manipulation problems.

\begin{figure}
	\centering
	\includegraphics[width=\linewidth]{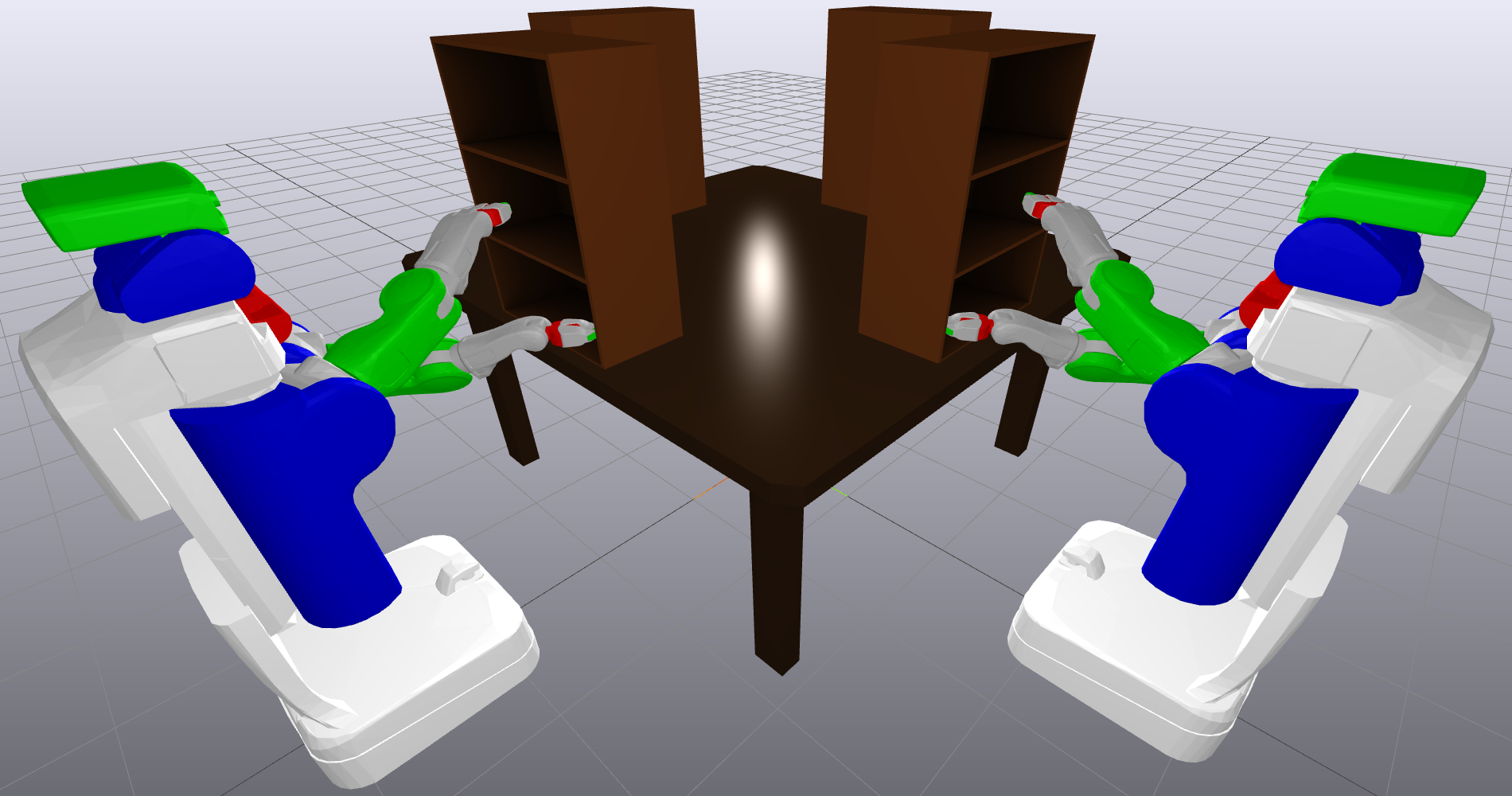}
	\caption{
		The start and goal pose for one of our motion planning experiments, using the PR2 bimanual mobile manipulator.
	}
	\label{fig:teaser}
\end{figure}

However, GCS-Planner is limited to Euclidean configuration spaces.
A mobile manipulator's configuration space is inherently non-Euclidean due to the mobile base: the robot can rotate through a full 360\textdegree, and its configuration is identical to when it started.
Continuous revolute joints present a similar issue.
Although the configuration spaces of interest are inherently non-Euclidean, they are still ``locally'' Euclidean, leading to elegant descriptions as differentiable manifolds.
With a Riemannian metric, which allows one to measure distance on a manifold, the concepts of convexity generalize to nonlinear spaces.
This in turn allows optimization on manifolds with rigorous guarantees, analogous to those obtained from convex optimization on Euclidean spaces.
However, this does preclude us from considering nonconvex costs and constraints, such as those derived from nonlinear dynamics and closed kinematic chains.
We also cannot handle ball joints, as their configuration spaces lead to unavoidable nonconvexities, which we describe in \Cref{subsec:positive_curvature_nonconvexity}.

In this paper, we formulate the general problem of shortest-path motion planning around obstacles on Riemannian manifolds.
We define a graph of \emph{geodesically-}convex sets (GGCS), the analogue to GCS on a Riemannian manifold.
We prove that this formulation has all the requisite properties needed to inherit the same guarantees as (Euclidean) GCS.
We then turn our attention to a certain class of robot configuration spaces, encompassing open kinematic chains with continuous revolute joints and mobile bases.
We show that in this case, our theoretical developments lead to simple and elegant modifications to the original GCS-Planner.
We entitle this generalization \emph{GGCS-Planner}, and demonstrate its efficacy with several challenging motion planning experiments.

\section{Related Work}
\label{sec:related_work}

In the world of continuous motion planning around obstacles, most popular techniques fall into two categories: sampling-based planners and trajectory optimizers.

Sampling-based motion planners partially cover C-Free with a large number of discrete samples.
Two of the foundational sampling-based planning algorithms are Probabilistic Roadmaps (PRMs)~\cite{kavraki1996probabilistic} and Rapidly-Exploring Random Trees (RRTs)~\cite{lavalle1998rapidly}.
Such algorithms are probabilistically complete, i.e., with enough samples, they will always find a valid path (if one exists).
However, these algorithms are only effective if a valid plan can be produced with a reasonable number of samples.
Hence, the ``curse of dimensionality'' is a potential obstacle to sampling-based planning, and such techniques have struggled with high-dimensional problems such as bimanual manipulation.
In most cases, planning for bimanual tasks is accomplished by planning for one arm, then planning the second arm independently while treating the first arm as a dynamic obstacle.
This is a reasonable heuristic for some tasks, but it sacrifices even probabilistic completeness.

An alternative approach is to formulate motion planning as an optimization problem.
This requires parametrizing the space of all trajectories and defining constraints and cost functions that describe the suitability of each trajectory.
Examples of kinematic trajectory optimization include B-spline parametrizations using constrained optimization~\cite[\S7.2]{tedrake2022manipulation}, CHOMP~\cite{zucker2013chomp}, STOMP~\cite{kalakrishnan2011stomp}, and KOMO~\cite{toussaint2014newton}.
Trajectory optimization approaches do not suffer from the curse of dimensionality, and are suitable for much more complex robotic systems.
But the optimization landscape is inherently nonconvex, so trajectory optimization methods often cannot achieve global optimality and can often fail to produce feasible trajectories when solutions exist.

The use of mixed integer programming (MIP) to solve motion planning problems to global optimality has recently seen an increase in popularity as new theoretical results, greater computational resources, and powerful commercial solvers~\cite{gurobi,mosek} have been brought to bear.
The survey paper of \citet{ioan2021mixed} provides an overview of the use of MIP in motion planning.
Besides the work of \citet{marcucci2021shortest}, \citet{landry2016aggressive} used MIP to plan aggressive quadrotor flights through obstacle-dense environments.
MIP has been used to plan footstep locations for humanoid robots~\cite{deits2014footstep} and for quadrupeds~\cite{valenzuela2016mixed,aceituno2017simultaneous}.
\citet{dai2019global} used MIP to globally solve the inverse kinematics problem.
Finally, MIP has seen extensive use in hybrid task and motion planning~\cite{adu2022optimal,tika2022robot,saha2017task,yi2022online,chen2022cooperative}.

Mixed integer programs can take a long time to solve in the worst case, but it is often possible to mitigate this problem with appropriate relaxations or approximations~\cite{suh2020energy,marcucci2022motion}.
GCS in particular uses an MIP formulation with a small number of integer variables, making branch-and-bound tractable.
Furthermore, the convex relaxation is tight, enabling efficient approximation by solving only a convex problem combined with a randomized rounding strategy. \cite{marcucci2022motion} argued that for single-arm manipulators, this approach can find more optimal plans in less time than PRMs.
These valuable properties carry over to our extension of GCS.

Another recent trend in motion planning has been the use of Riemannian geometry to model the problem.
Riemannian Motion Policies (RMPs)~\cite{ratliff2018riemannian} combine acceleration-based controllers across different task spaces into a single unified controller.
A Riemannian metric in each task space determines the priority of a given controller, and smooth maps between the manifolds enable the averaging of controllers.
RMPs have seen continued improvement~\cite{cheng2018rmpflow,rana2021towards} and generalization~\cite{van2022geometric,bylard2021composable}.
\citet{klein2022riemannian} envision Riemannian geometry as a tool for generating and combining elegant motion synergies for complex robotic systems.

\section{Preliminaries}
\label{sec:preliminaries}

In this section, we cover some of the relevant mathematical background.
We supply intuitive definitions; for further reference on Riemannian geometry, see the textbooks of \citet{docarmo1992riemannian} and \citet{lee2013smooth,lee2018introduction}.
\citet{boumal2020introduction} provides an excellent treatment of optimization over manifolds.
We use the notation $[n]=\set{1,\ldots,n}$.

\subsection{Riemannian Geometry}

A $d$-dimensional (topological) \emph{manifold} $\mc{M}$ is a locally Euclidean topological space: for any $p\in\mc{M}$, there is an open neighborhood $\mc{U}$ of $p$ and a continuous map $\psi:\mc{U}\to\R^{d}$ which is a homeomorphism onto its image.
The pair $(\mc{U},\psi_\mc{U})$ is called a \emph{coordinate chart}, and for any pair of overlapping charts $(\mc{U},\psi_\mc{U}^{\vphantom{1}})$ and $(\mc{V},\psi_\mc{V}^{\vphantom{1}})$, we have a \emph{transition} map
\begin{equation}
	\tau_{\mc{U},\mc{V}}^{\vphantom{1}}=\restr{\psi_\mc{V}^{\vphantom{1}}\of \psi_\mc{U}\inv}_{\psi_\mc{U}^{\vphantom{1}}(\mc{U}\cap \mc{V})}
\end{equation}
A collection of charts whose domains cover the manifold is called an \emph{atlas}.
We only consider $\mc{C}^{\infty}$-\emph{smooth} manifolds, where all transition maps in the atlas are $\mc{C}^{\infty}$.

For each $p\in\mc{M}$, the \emph{tangent space} $\mc{T}_p\mc{M}$ is a $d$-dimensional vector space representing the set of directional derivatives at $p$.
Given a differentiable curve $\gamma:(-\epsilon,\epsilon)\to\mc{M}$ with $p=\gamma(0)$, this affords an interpretation of the \emph{velocity} of $\gamma$ at $p$, $\gamma'(0)$, as an element of $\mc{T}_p\mc{M}$.
For a smooth map of manifolds $f:\mc{M}\to\mc{N}$, the \emph{pushforward} of $f$ at $p$ is a linear map $f_{*,p}:\mc{T}_p\mc{M}\to\mc{T}_{f(p)}\mc{N}$, generalizing the Jacobian matrix~\cite[p. 55]{lee2013smooth}.
The pushforward is defined so that, with $\gamma$ defined as above, $f_{*,p}(\gamma'(0))=(f\of\gamma)'(0)$.

A \emph{Riemannian metric} $g$ is a smoothly-varying positive-definite bilinear form over $\mc{M}$ that gives each tangent space $\mc{T}_p\mc{M}$ an inner product $\iprod{\;\cdot\;,\;\cdot\;}_{p}^{(\mc{M},g)}$.
The pair $(\mc{M},g)$ is a \emph{Riemannian manifold}, and we frequently refer to $\mc{M}$ exclusively when the choice of metric is unambiguous.
A Riemannian metric allows one to measure the length of a curve, invariant to reparametrizations~\cite[p.~34]{lee2018introduction}; if $\gamma:[a,b]\to\mc{M}$ is piecewise continuously differentiable, then
\begin{equation}
L(\gamma)=\int_a^b\sqrt{\iprod{\gamma'(s),\gamma'(s)}_{\gamma(s)}^{(\mc{M},g)}}\,\rmd s
	\label{eq:arc_length}
\end{equation}
We call the integrand the \emph{speed} of $\gamma$.
The distance between any two points $p,q\in\mc{M}$ is defined as the infimum of the arc length of all curves connecting them:
\[
	\dist_{\mc{M}}(p,q)\!=\!\inf\set{L(\gamma)\middle|\gamma\in\mc{C}^1_{\piecewise}([0,1],\mc{M}),\gamma(0)\!=\!p,\gamma(1)\!=\!q}
\]
where $\mc{C}^{1}_{\piecewise}([0,1],\mc{M})$ is the set of parametric piecewise-continuously differentiable curves from the interval $[0,1]$ to $\mc{M}$.
A curve that achieves this infimum need not exist in general~\cite[p.~146]{docarmo1992riemannian}.
We also define $\dist_{\mc{U}}(p,q)$ for $p,q\in\mc{U}\subseteq\mc{M}$ to be the infimum of the length of paths whose image is contained in $\mc{U}$.

If $\mc{M}$ is connected, it is a metric space with respect to $\dist_{\mc{M}}$.
Given two Riemannian manifolds $(\mc{M},g)$ and $(\mc{N},h)$, a smooth function $f:\mc{M}\to\mc{N}$ is a \emph{local isometry} if
\[
	\iprod{u,v}_{p}^{(\mc{M},g)}=\iprod{f_{*,p}(u),f_{*,p}(v)}_{f(p)}^{(\mc{N},h)}
\]
$\forall p\in\mc{M}$, $\forall u,v\in\mc{T}_p\mc{M}$.
If $f$ is also a diffeomorphism, and $\mc{M}$ and $\mc{N}$ are connected, then $f$ preserves distances~\cite[p.~37]{lee2018introduction}, and is an \emph{isometry} of metric spaces.
The converse is also true~\cite{myers1939group}.

A \emph{geodesic} is a locally length-minimizing curve, parameterized to be constant speed.
Locally length-minimizing means that for two points on the geodesic that are close enough, the geodesic traces out the shortest path between them.
For example, geodesics in Euclidean space with the natural metric are lines, rays, and line segments, and geodesics on the sphere (with the induced metric from Euclidean space) are great circles.
Constructing the shortest geodesic between two points is a variational calculus problem, so the solution must satisfy the Euler-Lagrange system of differential equations.
Thus, initial conditions $p\in\mc{M}$ and $v\in\mc{T}_p\mc{M}$ uniquely define a geodesic, such that $v$ is the velocity of the geodesic as it passes through $p$.
This is used to define the \emph{exponential map} $\exp_p:\mc{T}_p\mc{M}\to\mc{M}$, where the direction of a vector $v$ defines the direction of the geodesic, and the magnitude of $v$ determines how far to move in that direction away from $p$.

A Riemannian metric induces \emph{curvature} on a manifold, capturing how local geometry differs from the standard Euclidean case.
The \emph{sectional curvature} at a point $p$ is a real-valued function defined on $2$-dimensional subspaces of the tangent space $\mc{T}_p\mc{M}$~\cite[\S4.3]{docarmo1992riemannian}.
(We write $\mc{K}(u,v)$ for any vectors $u$ and $v$ that span the subspace.)
Informally, the sectional curvature corresponds to the distortion of angles in triangles, as shown in \Cref{fig:sectional-curvature}.
Manifolds that have everywhere-zero curvature are called \emph{flat}, and are locally isometric to Euclidean space. %

The Cartesian product of two Riemannian manifolds is itself a Riemannian manifold.
The curvature of the component manifolds influences the curvature of the product.
Importantly, \emph{the product of flat manifolds is flat}~\cite{atceken2003product}.
As we explain in \Cref{sec:ggcs}, this implies that a robot with a mobile base and (potentially many) continuous revolute joints has a flat configuration space.

\begin{figure}[t]
	\centering
	\begin{subfigure}[b]{0.3\linewidth}
		\centering
		\includegraphics[width=\linewidth]{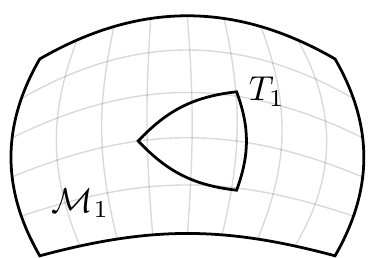}
		\caption{Positive}
		\label{subfig:sectional-curvature-positive}
	\end{subfigure}
	\hfill
	\begin{subfigure}[b]{0.3\linewidth}
		\centering
		\includegraphics[width=\linewidth]{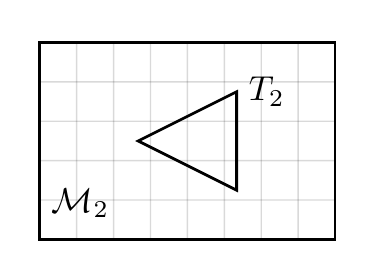}
		\caption{Zero (flat)}
		\label{subfig:sectional-curvature-zero}
	\end{subfigure}
	\hfill
	\begin{subfigure}[b]{0.3\linewidth}
		\centering
		\includegraphics[width=\linewidth]{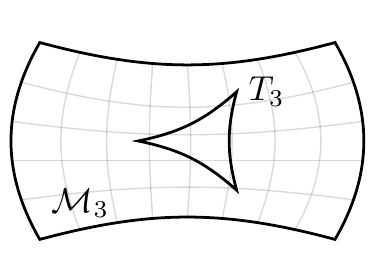}
		\caption{Negative}
		\label{subfig:sectional-curvature-negative}
	\end{subfigure}
	\caption{
		Examples of geodesic triangles $T_i$ in manifolds $\mc{M}_i$ with various sectional curvatures.
		In positive curvature spaces, the interior angles sum to more than 180\textdegree{}, and the edges bow outwards.
		The opposite is true in negative curvature spaces.
	}
	\label{fig:sectional-curvature}
\end{figure}

\subsection{Convex Analysis on Manifolds}

To define convexity on a Riemannian manifold $(\mc{M},g)$, we replace the notion of lines with geodesics.
In general, there is not a unique geodesic (or even a unique shortest geodesic) between two points, so a more intricate definition is required.
A subset $\mc{U}\subseteq\mc{M}$ is \emph{strongly geodesically convex} (or \emph{g-convex}) if $\forall p,q\in\mc{U}$, there is a unique length-minimizing geodesic connecting $p$ and $q$, and it is entirely contained in $\mc{U}$.
This definition ensures that the intersection of g-convex sets is g-convex, and that there is a unique shortest path between any pair of points in a g-convex set.
Weaker definitions used in other works~\cite{vishnoi2018geodesic,zhang2016first} do not provide these guarantees.
See \cite[\S11.3]{boumal2020introduction} for further discussion.

G-convex neighborhoods exist around every point \cite[p.~77]{docarmo1992riemannian}
For any $p\in\mc{M}$, there is a \emph{convexity radius} $r_p>0$, such that the open ball
\begin{equation}
	B_r(p)=\set{\exp_p(q)\setmid q\in\mc{T}_p\mc{M},\norm{q}<r}
\end{equation}
is g-convex for any $r<r_p$ (where the norm is induced by the Riemannian metric).
Intuitively, the convexity radius quantifies how large a set can be before minimizing geodesics can go ``the wrong way around'' the manifold.
On the product of two Riemannian manifolds, each geodesic is naturally the product of geodesics on its components.
Thus, the product of g-convex sets is g-convex in the product manifold.

A function $f:\mc{M}\to\R$ is said to be \emph{geodesically convex} (\emph{g-convex}) on $\mc{U}\subseteq\mc{M}$ if, for any geodesic $\gamma:[0,1]\to\mc{U}$, $(f\of\gamma)$ is a convex function on $[0,1]$.
That is, $\forall t\in[0,1]$,
\begin{equation}
	f(\gamma(t))\le (1-t)f(\gamma(0))+tf(\gamma(1))
\end{equation}
We say that $f$ is \emph{locally} g-convex if for any $p\in\mc{M}$, there exists a neighborhood $\mc{U}_p$ of $p$ such that the restriction of $f$ to $\mc{U}$ is g-convex.

Unfortunately, existing research into g-convex optimization often focuses on specific classes of manifolds that do not encompass the configuration spaces of interest~\cite{bacak2014convex,zhang2016first}.
In addition, there is little existing literature studying mixed-integer Riemannian convex optimization, and techniques commonly used in the Euclidean case (e.g., cutting planes~\cite{marchand2002cutting}) may not generalize to Riemannian manifolds.

\section{Problem Statement}
\label{sec:statement}

We may now precisely state our kinematic planning problem in the language of Riemannian geometry developed thus far.
Let $(\mc{Q},g)$ be the configuration space of a robot, realized as a connected Riemannian manifold\footnote{
	We also require $\mc{Q}$ to be \emph{complete}~\cite[p. 598]{lee2013smooth} w.r.t. the metric induced by $g$.
}.
Suppose that the set of collision-free configurations is a bounded open subset $\mc{M}\subseteq\mc{Q}$, and without loss of generality, assume that $\mc{M}$ is path-connected.
(If $\mc{M}$ is not path-connected, then we restrict ourselves to planning within a single connected component.)

Suppose we want to find the shortest path between two points $p$ and $q$ in $\closure{\mc{M}}$, the closure of $\mc{M}$ (i.e., the smallest closed set containing $\mc{M}$).
This can be written as the optimization problem
\begin{equation}
	\begin{array}{rl}
		\argmin & L(\gamma)\\
		\ptxt{\st} & \gamma \in \CpwM\\
		& \gamma(0)=p\\
		& \gamma(1)=q
	\end{array}
	\label{eq:motion_planning_formulation}
\end{equation}
where $L$ is the Riemannian arc length, given in \Cref{eq:arc_length}.
In the following sections, we develop machinery to solve optimization problems of this form.

\section{Graphs of Geodesically-Convex Sets}
\label{sec:ggcs}

We now introduce a \emph{graph of geodesically convex sets} (GGCS), a Riemannian optimization framework that, in \Cref{sec:motion_planning}, we show is general enough to encompass Problem~\eqref{eq:motion_planning_formulation}.
A GGCS is a directed graph $G=(V,E)$ with certain properties, designed as a generalization of ordinary (Euclidean) GCS from \citet[\S 2]{marcucci2021shortest} to Riemannian manifolds.
Each vertex $v\in V$ has a corresponding a g-convex subset $\mc{Y}_{v}$ of some Riemannian manifold $(\mc{M}_{v},g_v)$. 
With each edge $e=(u,v)\in E$, we associate a cost function $\ell^{\mc{Y}}_{e}:\mc{Y}_{u}\times\mc{Y}_{v}\to \R_{\ge 0}\cup\set{\infty}$, which must be g-convex with respect to the product metric on $\mc{M}_{u}\times\mc{M}_{v}$.
For all problems considered in this paper, every g-convex set will be a subset of a single Riemannian manifold.

Given distinct source and target vertices $p,q\in V$, a \textit{path} $\pi$ from $p$ to $q$ is a sequence of distinct vertices $(v_k)_{k=0}^{K}$ such that $v_0=p$, $v_K=q$, and $(v_{k-1},v_k)\in E$ for all $k\in[K]$.
Associate to this path a sequence of points $y_\pi=(y_0,\ldots,y_K)$ such that each $y_v\in\mc{Y}_{v}$; then the length of this path is
\begin{equation}
	\ell^{\mc{Y}}_\pi(y_\pi)=\sum_{k=1}^{K}\ell^{\mc{Y}}_{(v_{k-1},v_k)}(y_{k-1}, y_k)
	\label{eq:ggcs_path_length}
\end{equation}
Let $\Pi$ denote the set of all paths from $p$ to $q$, and for any $\pi\in\Pi$, define its feasible vertices as $\mc{Y}_{\pi}=\mc{Y}_{v_{0}}\times\cdots\times\mc{Y}_{v_{K}}$. The problem of finding the shortest path from $p$ to $q$ can be written as
\begin{equation}
	\underset{\pi\in\Pi}{\min}\,\underset{y_\pi\in\mc{Y}_{\pi}}{\min}\ell_{\pi}^{\mc{Y}}(y_{\pi})
	\label{eq:ggcs_compact}
\end{equation}

\begin{figure}
	\centering
	\includegraphics[width=\linewidth]{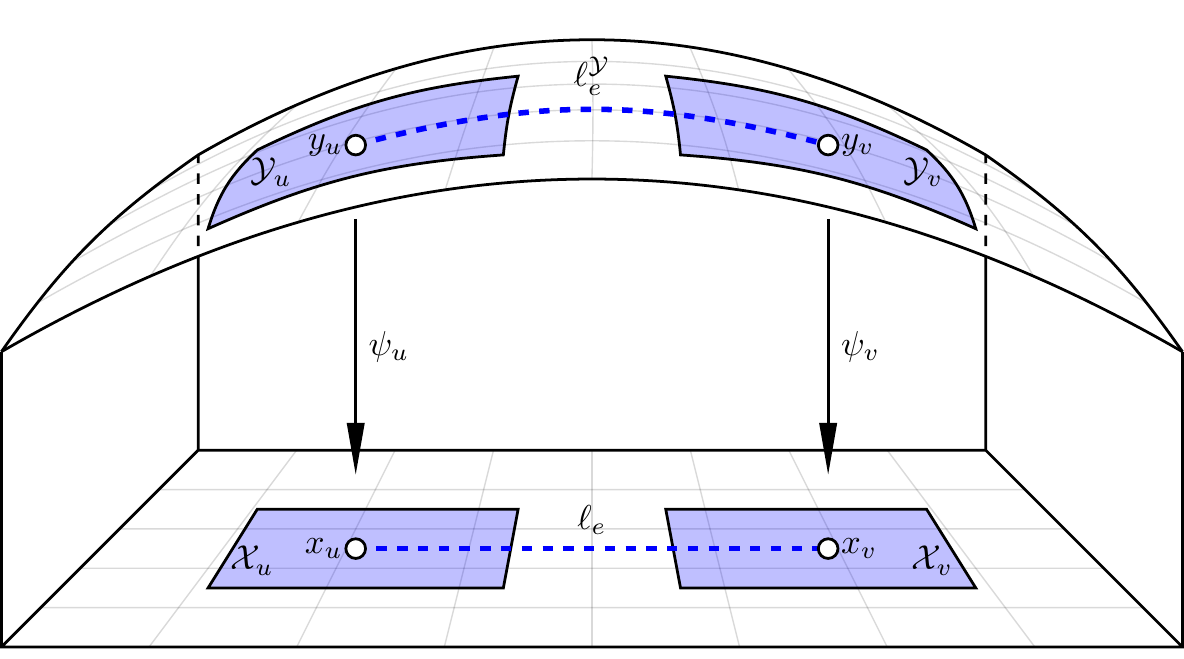}
	\caption{Moving edges and sets from Riemannian manifolds to Euclidean spaces with coordinate charts.
	In this diagram, $\mc{Y}_{u}$ and $\mc{Y}_{v}$ are visualized as part of the same Riemannian manifold, although this need not be true in general.}
	\label{fig:edge-pullback}
\end{figure}

Solving Problem~\eqref{eq:ggcs_compact} to optimality is intractable in complete generality, so we propose to transform it into an ordinary GCS problem.
To each $v\in V$, associate a chart $\psi_v$, and define $\mc{X}_{v}=\psi_v(\mc{Y}_{v})$.
For each edge $e=(u,v)\in E$, we define the edge cost on $(x_u,x_v) \in \mc{X}_u \times \mc{X}_v$:
\begin{equation}
	\ell_{e}(x_u,x_v)=\ell_{e}^{\mc{Y}}(\psi_u\inv(x_u),\psi_v\inv(x_v)).
	\label{eq:pullback_edge_cost}
\end{equation} 
This construction is shown in \Cref{fig:edge-pullback}.
To apply the GCS machinery, we require that the sets $\mc{X}_v$ and edge costs $\ell_{e}(x_u,x_v)$ are convex.
As we show in \Cref{subsec:positive_curvature_nonconvexity}, this is hopeless for manifolds with positive curvature.
Luckily, for flat manifolds, convexity can be ensured, as will be shown in \Cref{subsec:formulation_as_gcs}.

Importantly, many robot configuration spaces can be realized as flat manifolds.
$\operatorname{SE}(2)$ is flat, all $1$-dimensional manifolds are flat~\cite[p.~47]{lee2018introduction} (this encompasses continuous revolute joints), and products of flat manifolds are flat.
Thus, any robotic system whose configuration can be described using a series of single-degree-of-freedom joints (and potentially a mobile base) will have a flat configuration space, and thus can be handled by our methodology.
2-DoF universal joints can also be handled, as they can be perfectly represented as two juxtaposed 1-DoF joints.
3-DoF ball joints cannot be handled, because decomposing a ball joint into 1-DoF joints distorts the underlying geometry.

\section{Motion Planning with GGCS}
\label{sec:motion_planning}

We want to use GGCS to make motion plans on Riemannian manifolds by solving Problem~\eqref{eq:motion_planning_formulation}.
Thus, we must prove that the optimal value is achieved by some trajectory that is feasible for a GGCS problem.
We use the initialism ROSC (Riemannian Open Subset Closure) to describe closures of open subsets of Riemannian manifolds, notably $\closure{\mc{M}}$.
ROSCs are topological manifolds-with-boundary, but the boundary may not be smooth; for example, polytopic obstacles lead to corners on the boundary of $\closure{\mc{M}}$.
The theory of manifolds-with-corners is not well developed in full generality, so for the sake of completeness, we confirm some expected properties of paths through ROSCs.

\begin{theorem}
	\textbf{(Existence of Optimal Trajectories)}
	For any $p,q\in\closure{\mc{M}}$, there exists a continuous curve $\gamma$ connecting them such that $L(\gamma)=\dist(p,q)$.
	\label{thm:burago}
\end{theorem}
\begin{proof}
	The proof follows by verifying that $\closure{\mc M}$ satisfies the preconditions of Theorem 2.5.23 of \citet[p.~50]{burago2022course}.
	We defer the details to \Cref{proof:burago}.
\end{proof}

\begin{assumption}
	We are given a finite atlas $\mc{A}=\set{(\mc{Y}_v,\psi_v)}_{v\in V}$ of $\mc{M}$.
	For each $v$, the closure $\closure{\mc{Y}}_v$ is g-convex as a subset of $\mc{Q}$.
	Furthermore, the union of the closures $\closure{\mc{Y}}_{v}$ covers $\closure{\mc{M}}$.
	\label{asm:mcYv}
\end{assumption}

These requirements will not hold in general, but we will discuss how to construct such an atlas in \Cref{subsec:atlas_construct}.
We can also extend each $\psi_v$ to be defined on $\closure{\mc{Y}}$.
With this information, we can prove a strong result about the shortest paths in $\closure{\mc{M}}$.

\begin{theorem}
	\textbf{(Piecewise Geodesic Optimal Paths)}
	Let $p,q\in\closure{\mc{M}}$, and suppose the sets $\closure{\mc Y}_v$ satisfy \Cref{asm:mcYv}.
	Then there is a curve $\gamma^*\in \Cpw([a,b],\closure{\mc{M}})$ connecting them, such that the following are true:
	\begin{itemize}
		\item $L(\gamma^*)=\dist(p,q)$
		\item $\gamma^*$ is a piecewise geodesic of $\mc{Q}$
		\item Each geodesic segment is contained in some $\closure{\mc{Y}}_{v}$
		\item $\gamma^*$ passes through each $\closure{\mc{Y}}_{v}$ at most once.
	\end{itemize}
	\label{thm:geodesic_segment}
\end{theorem}
\begin{proof}
	Let $\gamma_0$ be a continuous minimizing path connecting $p$ to $q$ (guaranteed to exist by \Cref{thm:burago}); we will use this to construct an appropriate $\gamma^*$.
	Select an arbitrary order $v_1,\ldots,v_{\card{V}}$ to iterate over all of the vertices in $V$.
	We will construct a sequence of curves $\gamma_1,\ldots,\gamma_{\card{V}}$, such that $\gamma_{\card{V}}$ has the desired properties.

	For each $i$, if $\gamma_{i-1}$ does not pass through $\closure{\mc{Y}}_{v_i}$, let $\gamma_i=\gamma_{i-1}$.
	Otherwise, let $T_i=\set{t\setmid \gamma_{i-1}(t)\in\closure{\mc{Y}}_{v_{i}}}$, let $a_i'=\min(T_i)$, and let $b_i'=\max(T_i)$.
	Then by the g-convexity of $\closure{\mc{Y}}_{v_i}$, there is a unique minimizing geodesic $\alpha_i:[a_i',b_i']\to\closure{\mc{Y}}_{v_i}$ connecting $\gamma_{i-1}(a_i')$ and $\gamma_{i-1}(b_i')$.
	Let $\tilde\gamma$ be a new curve, defined by
	\begin{equation}
		\tilde\gamma(t)=\begin{cases}
			\gamma_{i-1}(t) & t\not\in[a_i',b_i']\\
			\alpha_i(t) & t\in[a_i',b_i']
		\end{cases}
		\label{eq:better_curve}
	\end{equation}
	Because $L(\alpha_i)\le L(\restr{\gamma_{i-1}}_{[a',b']})$, we have $L(\tilde\gamma)\le L(\gamma_{i-1})$, and since $\gamma_{i-1}$ is of minimum length, we must have $L(\tilde\gamma)=L(\gamma_{i-1})$.
	Define $\gamma_{i}$ to be $\tilde\gamma$, and continue until we have iterated over all of the $v\in V$.
	Then by construction, $L(\gamma_{\card{V}})=\dist(p,q)$, $\gamma_{\card{V}}$ is piecewise geodesic in $\mc{Q}$, each geodesic segment is contained in some $\closure{\mc{Y}}_{v}$, and $\gamma_{\card{V}}$ passes through each $\closure{\mc{Y}}_{v}$ at most once.
\end{proof}

\subsection{Formulation as a GCS Problem}
\label{subsec:formulation_as_gcs}

To transform the GGCS problem into a GCS problem, we require that the sets and edge costs are convex in Euclidean space.
The following is sufficient (and still encompasses robots with mobile bases and continuous revolute joints):

\begin{assumption}
	$\mc{Q}$ is flat.
	Also, each $\psi_v$ is a local isometry into Euclidean space, viewed as a Riemannian manifold with the canonical Euclidean metric.
	\label{asm:flat}
\end{assumption}

\Cref{asm:mcYv,asm:flat} together yield three important results:
\begin{itemize}
	\item $\mc{X}_v=\psi_v(\closure{\mc{Y}}_v)$ is convex.
	\item $\forall y_0,y_1\in\closure{\mc{Y}}_v$, $\dist(y_0,y_1)=\norm{\psi_v(y_0)-\psi_v(y_1)}_{2}$
	\item $\tau_{u,v}$ is a Euclidean isometry (see \Cref{lemma:off_by_an_isometry} in \Cref{proof:convex_product}), and hence affine~\cite{vaisala2003proof}. 
\end{itemize}

The first two results are true because $\mc{Y}_v$ is g-convex, $\psi_v$ maps geodesics to geodesics~\cite[p.~125]{lee2018introduction}, and geodesics are unique in Euclidean space.
For most robotic configuration spaces we consider, $\mc{Q}$ can be decomposed as the product of one-dimensional manifolds.
In this case, the coordinate systems can be globally aligned, so that every transition map is a translation.

To formulate the problem with GCS, we follow an approach similar to~\cite{marcucci2022motion}, where decision variables describe line segments contained within each convex set.
In particular, $\forall v\in V$, we have $x_v=(x_{v,0},x_{v,1})\in\mc{X}_v^2$, where $x_{v,0}$ is the start point of the line segment, and $x_{v,1}$ is the end point.
For each edge $e=(u,v)\in E$, the length of the segment associated with the starting vertex is used as the edge cost:
\begin{equation}
	\ell_e(x_u,x_v)=\dist(\psi_u\inv(x_{u,0}),\psi_u\inv(x_{u,1}))=\norm{x_{u,0}-x_{u,1}}_{2}
	\label{eq:edge_cost}
\end{equation}
We also encode equality constraints to ensure the endpoints of adjacent segments are in agreement:
\begin{equation}
	\psi_u\inv(x_{u,1})=\psi_v\inv(x_{v,0})
	\quad\Leftrightarrow\quad
	\tau_{u,v}(x_{u,1})=x_{v,0}
	\label{eq:edge_constraint}
\end{equation}
This constraint is convex because $\tau_{u,v}$ is affine.
Thus, we have a valid GCS formulation, which can be solved as a mixed-integer convex program.
Alternatively, it can be solved approximately by solving the convex relaxation and using a randomized rounding strategy~\cite{marcucci2022motion}.
If we have $p\in \mc{X}_{0}$ and $q\in \mc{X}_{K}$, then after solving the GCS problem, we obtain a path
\begin{equation}
	x_\pi=(x_{0,0},x_{0,1},x_{1,0},x_{1,1},\ldots,x_{K,0},x_{K,1})
	\label{eq:euclidean_path}
\end{equation}
with $x_{0,0}=\psi_0(p)$, $x_{K,1}=\psi_K(q)$, and $\psi_i(x_{i,1})=\psi_{i+1}(x_{i+1,0})$, $\forall i\in\set{1,\ldots,K-1}$.
Such a path naturally lifts to a path on $\closure{\mc{M}}$:
\begin{align}
	\label{eq:manifold_path}
	y_\pi & =(y_0=p,y_1,y_2,\ldots,y_K,y_{K+1}=q)\\
	& =(\psi_0\inv(x_{0,0}),\psi_1\inv(x_{1,0}),\ldots,\psi_K\inv(x_{K,0}),\psi_K\inv(x_{K,1}))\nonumber
\end{align}
where we have removed duplicate points from the trajectory.
This process is visualized for a simple cylinder manifold in \Cref{fig:cylinder-example}.

\begin{figure}
	\centering
	\includegraphics[width=\linewidth]{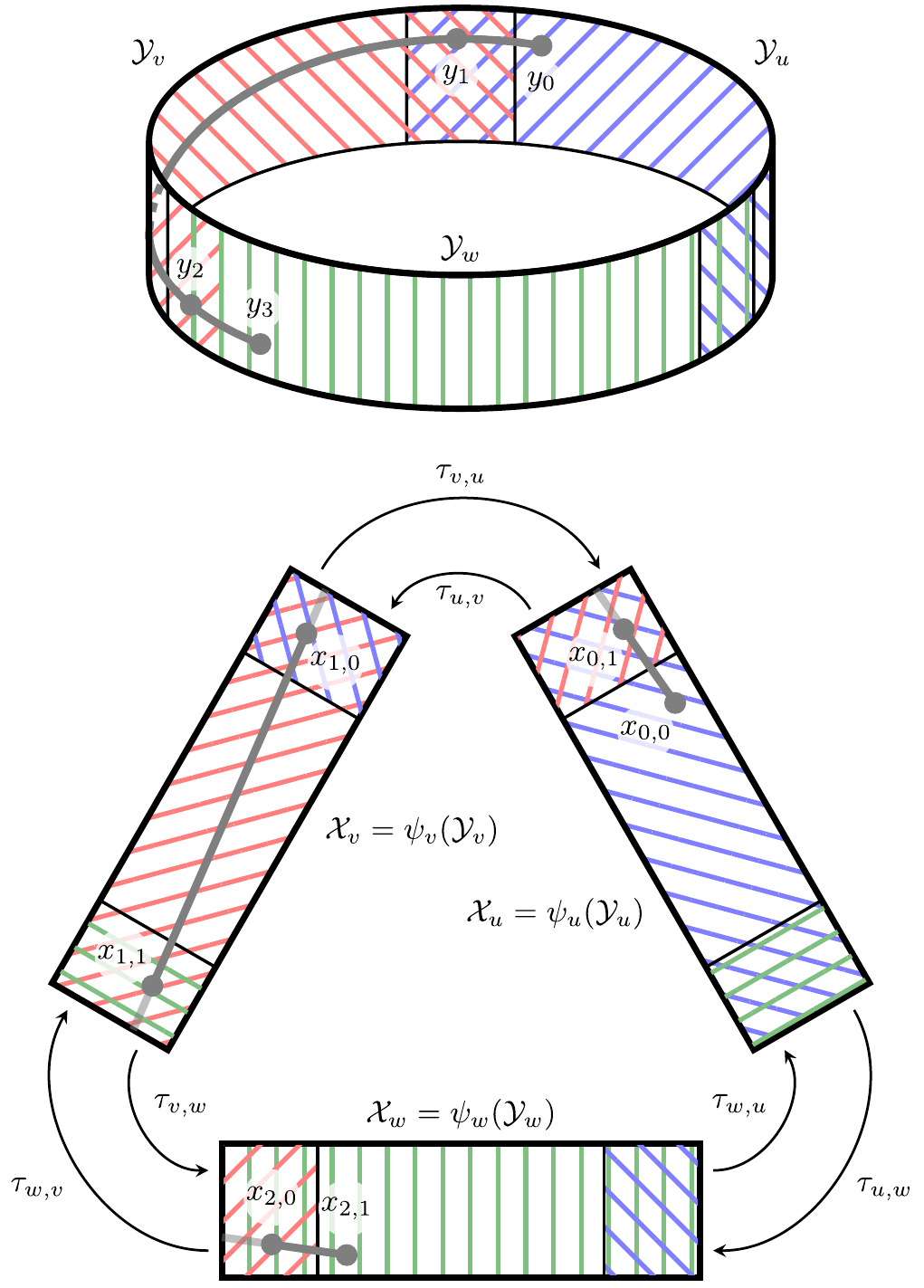}
	\caption{
		The process of transforming a GGCS problem into a GCS problem for a simple cylinder manifold.
		Each of the three charts maps to a Euclidean space, with transition maps encoding the equality constraints across chart domains.
		The line segments then lift to a piecewise geodesic on the manifold.
	}
	\label{fig:cylinder-example}
\end{figure}

For each $i\in\set{0,\ldots,K}$, $y_i$ and $y_{i+1}$ are contained in a g-convex set $\closure{\mc{Y}}_i$, so there is a unique minimizing geodesic $\gamma_i$ connecting them and completely contained in $\closure{\mc{Y}}_i$.
Thus, a path $y_\pi$ uniquely defines a piecewise geodesic $\gamma_\pi$ connecting $p$ to $q$ that is completely contained in $\closure{\mc{M}}$.
With this fact, we can formally prove the equivalence of the GGCS problem and the GCS problem.
\begin{theorem}
	\textbf{(Proof of Problem Equivalence)}
	If the path $x_\pi$ given in \Cref{eq:euclidean_path} is optimal for the GCS problem defined by \Cref{eq:edge_cost,eq:edge_constraint}, then the piecewise geodesic $\gamma_\pi$ defined by \Cref{eq:manifold_path} is optimal for Problem~\eqref{eq:motion_planning_formulation}.
	\label{thm:problem_equivalence}
\end{theorem}
\begin{proof}
	Any feasible path $x_\pi$ for the GCS problem yields a piecewise continuously differentiable curve $\gamma_\pi$ whose image is contained in $\closure{\mc{M}}$ and connecting $p$ to $q$.
	Then the length of this curve satisfies
	\[
		L(\gamma_\pi)=\sum_{i=0}^{K}\dist(y_i,y_{i+1})=\sum_{i=0}^{K}\norm{x_{i,0}-x_{i,1}}_2=\ell_{\pi}(x_{\pi})
	\]
	Thus, the optimal value of Problem~\eqref{eq:motion_planning_formulation} is no worse than the optimal value of the GCS problem.

	Now, consider an optimal $\gamma^*$ for Problem~\eqref{eq:motion_planning_formulation}, with the properties of \Cref{thm:geodesic_segment}.
	Then $\gamma^*$ is the concatenation of geodesics $\gamma_1,\ldots,\gamma_K$, where $\gamma_i:[0,1]\to\closure{\mc{Y}}_{v_{i}}$ for $i=1,\ldots,K$, and each $v_i$ is distinct.
	Define $x_\pi$ by
	\[
		(x_{i,0},x_{i,1})=\paren{\psi_{i}(\gamma_{i}(0)),\psi_{i}(\gamma_i(1))}
	\]
	$\forall i\in[K]$.
	By construction, $\ell_\pi(x_\pi)=L(\gamma^*)$.
	$\gamma_i(1)=\gamma_{i+1}(0)$, and the $v_i$ are distinct, so $x_\pi$ is feasible for the GCS problem.
	Thus, the GCS problem achieves the optimal value of Problem~\eqref{eq:motion_planning_formulation}.
\end{proof}

\subsection{Construction of the Atlas}
\label{subsec:atlas_construct}

A key part of motion planning with GGCS is the construction of an appropriate atlas $\mc{A}=\set{(\mc{Y}_{v},\psi_v)}_{v\in V}$ of $\mc{M}$.
Recall that $\mc{A}$ must be finite, each $\closure{\mc{Y}}_{v}$ must be g-convex, and each $\psi_v$ must be a local isometry.

As was done in~\cite{marcucci2022motion}, we construct an inner approximation of C-Free using the extension of IRIS~\cite[Alg. 2]{amice2023finding} to handle nonconvex obstacles.
Given a seed point in $\mc{Q}$, we grow a region about that point with respect to a local coordinate system.
In this way, we grow the region in Euclidean space, but we have an implicit mapping to the manifold, allowing us to construct the transition maps needed for \Cref{eq:edge_constraint}.

To ensure the set is g-convex when lifted to $\mc{Q}$, we bound the region by the convexity radius on a per-joint basis.
If $r_i$ is the convexity radius of the $i$th joint's configuration space, we constrain that joint to take values within an open ball of radius $r_i$, centered at the seed point.
(Computationally, we use a closed ball of radius $r_i-\epsilon$, with a small $\epsilon>0$.)
For a 1DoF joint, this is just the interval $[x-r_i+\epsilon,x+r_i-\epsilon]$ for seed point $x$.
If the manifold is flat, this guarantees g-convexity (see proof in \Cref{proof:convex_product}).

\begin{theorem}
	Suppose $\mc{Q}=\mc{Q}_1\times\cdots\times\mc{Q}_m$, where each $\mc{Q}_i$ has a convexity radius $r_i$.
	Let $(\mc{Y},\psi)$ be a coordinate chart, with $\psi$ a local isometry and $\psi(\mc{Y})$ convex in Euclidean space.
	If $\mc{Q}$ is flat and the diameter of $\proj_{\mc{Q}_i}(\mc{Y})$ is less than $2r_i$, then $\mc{Y}$ is g-convex.
	\label{thm:convex_product}
\end{theorem}

We also assumed full coverage of $\closure{\mc{M}}$ by the union of the $\mc{Y}_{v}$.
In scenarios where we only have an inner approximation of C-Free, we treat all points outside of that approximation as obstacles.
Thus, our planner finds the globally optimal path within ``certified'' C-Free, which is a subset of the whole C-Free.

\subsection{More General Motion Planning}
\label{subsec:further_motion_planning}

\Citet{marcucci2022motion} extended GCS-Planner to parametrize trajectories as piecewise \bz{} curves, in order to handle a greater variety of costs and constraints.
This includes penalizing the duration and energy of a trajectory, adding velocity bounds, and requiring the trajectory to be differentiable a certain number of times.
\bz{} curves generalize naturally to Riemannian manifolds by interpolating along the minimizing geodesics between control points~\cite{park1995bezier,popiel2007bezier}.
Because we restrict ourselves to flat manifolds, the local geometry is unchanged from Euclidean space.
Thus, all costs and constraints that operate on individual segments of the piecewise \bz{} curve trajectory can be used with no changes.

To enforce the differentiability of the overall trajectory where two segments connect, we must compare tangent vectors across different coordinate systems.
In particular, suppose we need differentiability $\eta$ times for an edge $(i,j)$, with transition map $\tau_{i,j}$.
Let $\gamma_i$ and $\gamma_j$ be adjoining \bz{} curve segments in $\mc{Y}_i$ and $\mc{Y}_j$, and let their $k$th derivatives be $\nu_i^{(k)}$ and $\nu_j^{(k)}$ at the point where they connect, called $w$.
Using the pushforward of the transition map at $w$, this constraint can be written as
\begin{equation}
	(\tau_{i,j})_{*,\psi_i\inv(w)}\paren{\nu_i^{(k)}}=\nu_j^{(k)}\qquad\forall l\in[\eta]
	\label{eq:higher_order_continuity}
\end{equation}
Because the transition map is a Euclidean isometry, its pushforward is a linear transformation described by an orthogonal matrix, and if the coordinate systems are globally aligned (as described in \Cref{subsec:formulation_as_gcs}), then the pushforward is the identity map.
When $\mc{Q}$ is flat, the derivative of a \bz{} curve is a linear expression of its control points, so \Cref{eq:higher_order_continuity} is a convex constraint.

\subsection{Positive Curvature Induces Nonconvexity}
\label{subsec:positive_curvature_nonconvexity}

At this point, we have shown that the flatness of $\mc{Q}$ is sufficient for our formulation's validity.
It is natural to ask how essential this is, especially since $\SO(3)$, a manifold of great interest in robotics, has positive curvature.
For example, the configuration space of a ball joint is a subset of $\SO(3)$, and the configuration space of a free moving object in $\R^3$ is $\SE(3)\cong\SO(3)\times\R^3$.
Unfortunately, even a single point of positive curvature implies that the Riemannian distance function is not g-convex, even on arbitrarily small neighborhoods of that point (see proof in \Cref{proof:positive_nonconvexity}).
\begin{theorem}
	Let $\mc{M}$ be a Riemannian manifold, let $A_1\in\mc{M}$ and $u,v\in T_{A_1}\mc{M}$ such that $\mc{K}(u,v)>0$.
	Then for any neighborhood $\mc{U}$ containing $A_1$, $\dist:\mc{M}^2\to\R$ is nonconvex on $\mc{U}^2$.
	\label{thm:positive_nonconvexity}
\end{theorem}

\section{Experiments}
\label{sec:experiments}

We demonstrate our GGCS planner on various robotic platforms.
We present illustrative toy examples of planning for a point robot on a toroidal world and an arm in the plane with multiple continuous revolute joints.
We also build plans for a KUKA iiwa arm (with the base joint modified to be continuous revolute) and a PR2 bimanual mobile manipulator, implemented in Drake~\cite{tedrake2019drake}. 
We make interactive recordings of these trajectories available online at our \href{https://ggcs-anonymous-submission.github.io/}{results website}.
For each experiment, we explicitly state the configuration space, using $\mc{I}$ to refer to a general bounded interval in $\R$.

\subsection{Point Robot}
\label{sec:experiments:point_robot}

Consider a point robot moving about a toroidal world (configuration space $\mb{T}^2$, modeled as a unit square with the edges identified), with convex planar obstacles.
It is easy to visualize the obstacles, g-convex sets, and graph edges, as shown in \Cref{fig:point-robot}.
We also show an example of an optimal trajectory produced by GGCS-Planner, which ``wraps around'' the toroidal world.
This plan was computed in $0.79$ seconds.

\begin{figure}
	\centering
	\includegraphics[width=0.8\linewidth]{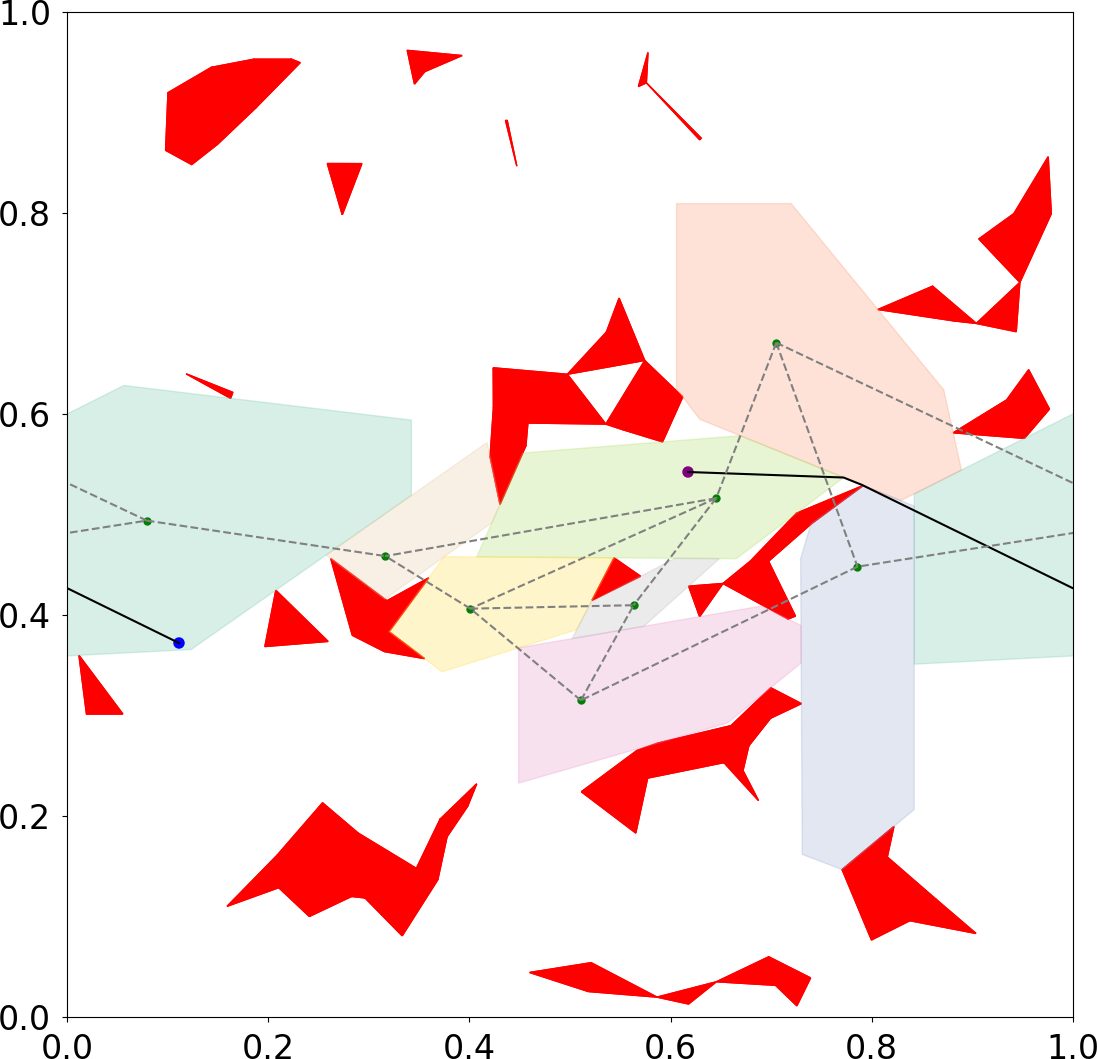}
	\caption{
		Results for a point robot in a toroidal world, realized as a unit square with opposite edges identified.
		Obstacles are shown in red, and each IRIS region is given a distinct pastel color.
		Note that one of the regions ``wraps around'' along the horizontal dimension, connecting opposite sides of the world.
		Grey dashed lines indicate which regions overlap.
		The optimal path between the start and end points is shown in black.
	}
	\label{fig:point-robot}
\end{figure}

\subsection{Planar Arm}
\label{sec:experiments:planar_arm}

Consider a robot arm with a fixed base, composed of five continuous revolute joints (configuration space $\mb{T}^5$), moving through a planar workspace with convex obstacles.
(We assume the arm does not suffer from self-collisions.)
We present sample plans produced by GGCS-Planner in \Cref{fig:planar-arm}, together with the swept collision volumes.
These two plans were found in $5.36$ and $4.63$ seconds, respectively.
A video of these trajectories is available at our \href{https://ggcs-anonymous-submission.github.io/}{results website}.

\begin{figure}
	\centering
	\begin{subfigure}[b]{\linewidth}
		\centering
		\includegraphics[width=\linewidth]{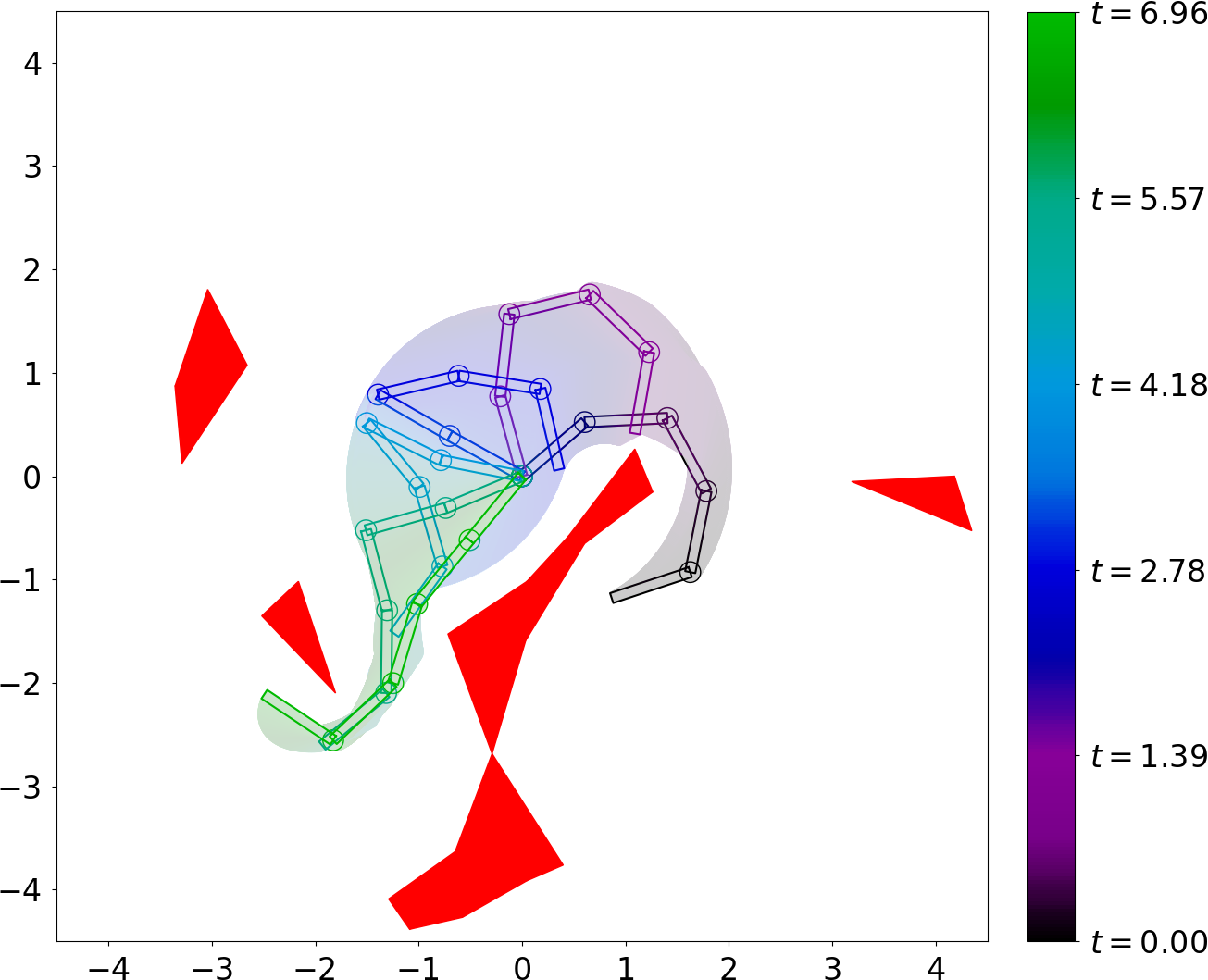}
	\end{subfigure}
	\break
	\begin{subfigure}[b]{\linewidth}
		\centering
		\includegraphics[width=\linewidth]{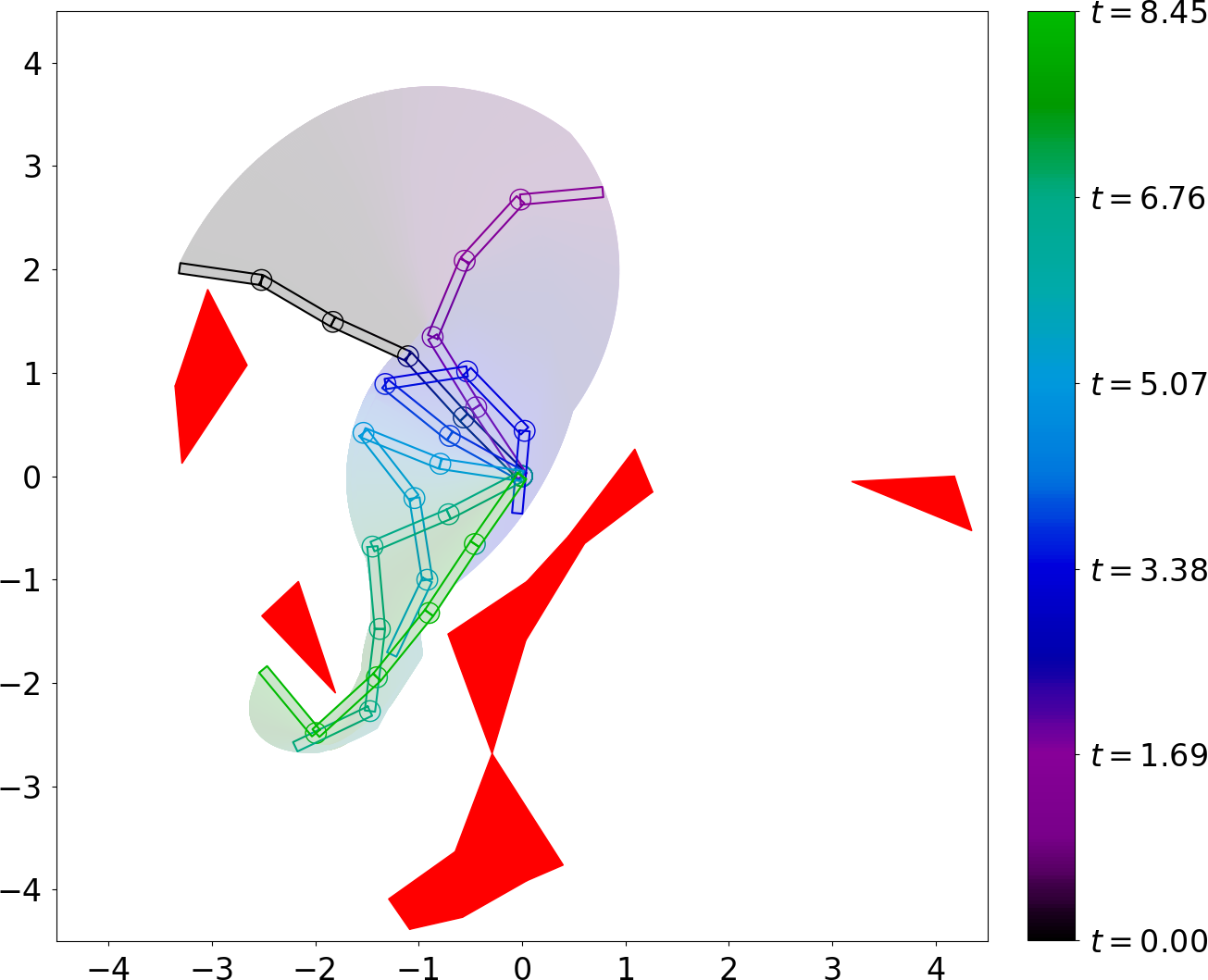}
	\end{subfigure}
	\caption{
		Two plans produced by GGCS-Planner for a planar arm around task-space obstacles (shown in red).
		We display both the swept collision volume and individual poses in the trajectory (colored by time, as indicated by the colorbar).
	}
	\label{fig:planar-arm}
\end{figure}

\subsection{Modified KUKA iiwa Arm}
\label{sec:experiments:kuka_iiwa}

\begin{figure}
	\centering
	\includegraphics[width=\linewidth]{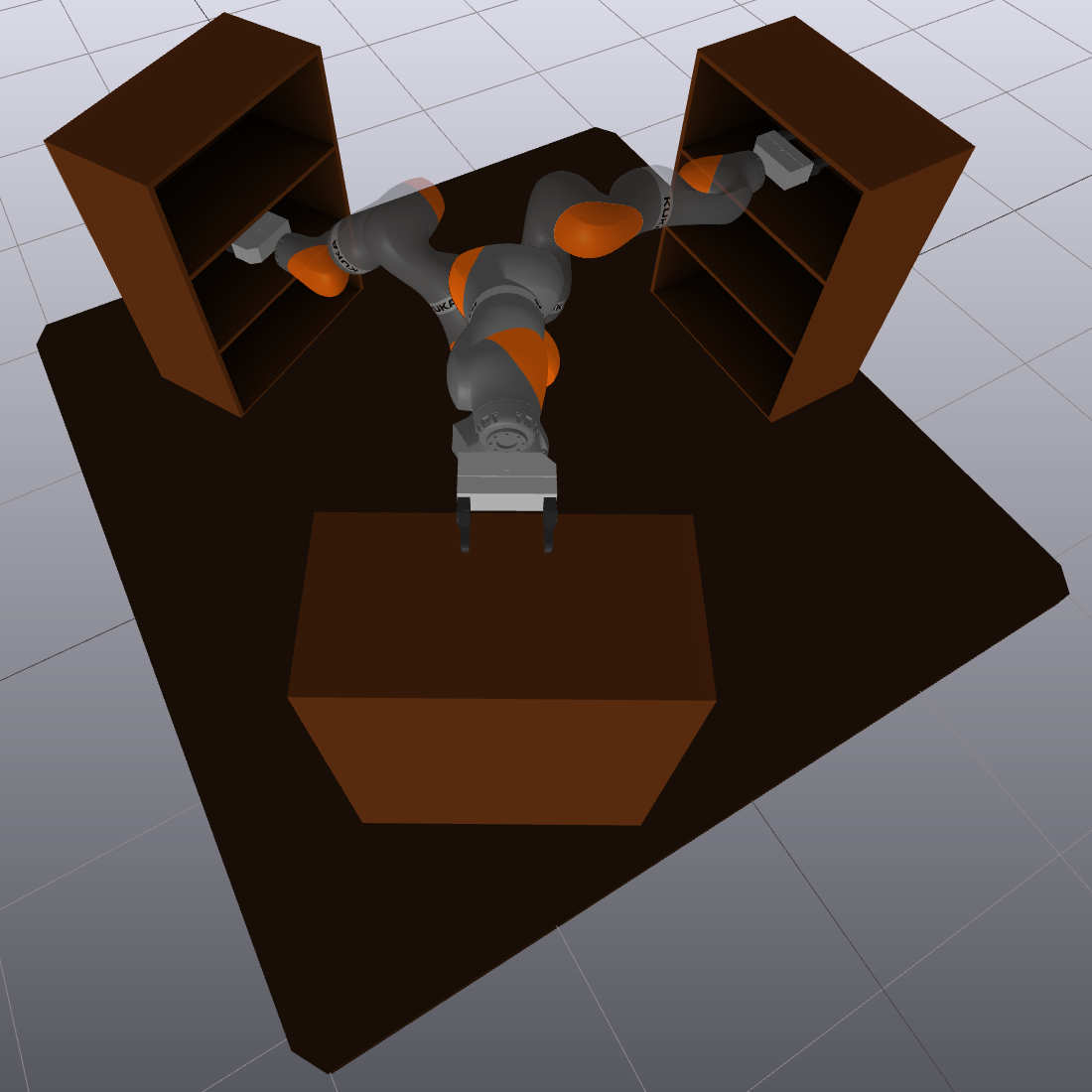}
	\caption{Key configurations (overlaid) used for a mug reorganization demo.}
	\label{fig:kuka_seeds}
\end{figure}

\begin{figure}
	\centering
	\begin{subfigure}[b]{0.49\linewidth}
		\centering
		\includegraphics[width=\linewidth]{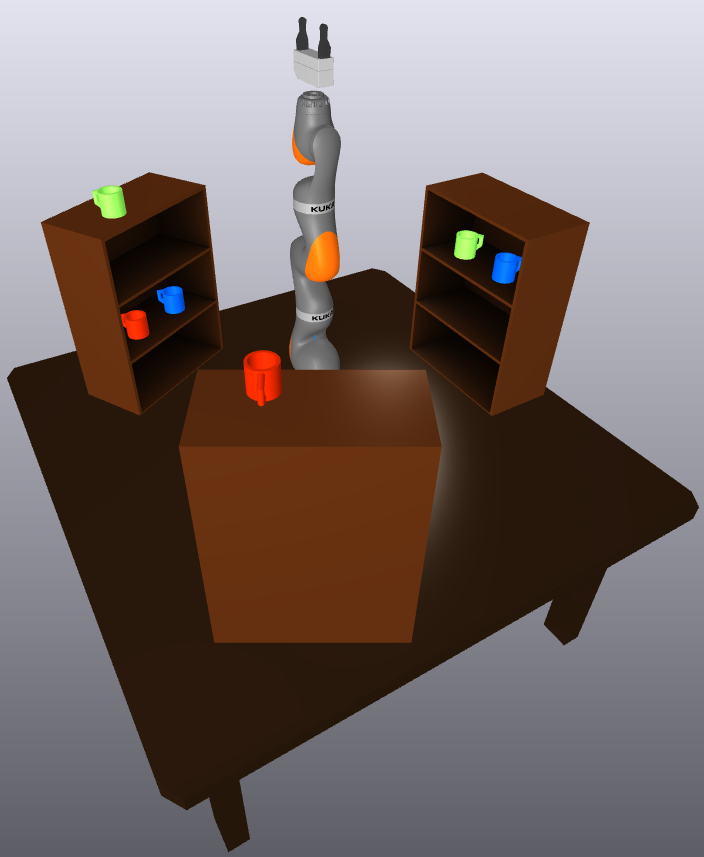}
	\end{subfigure}
	\hfill
	\begin{subfigure}[b]{0.49\linewidth}
		\centering
		\includegraphics[width=\linewidth]{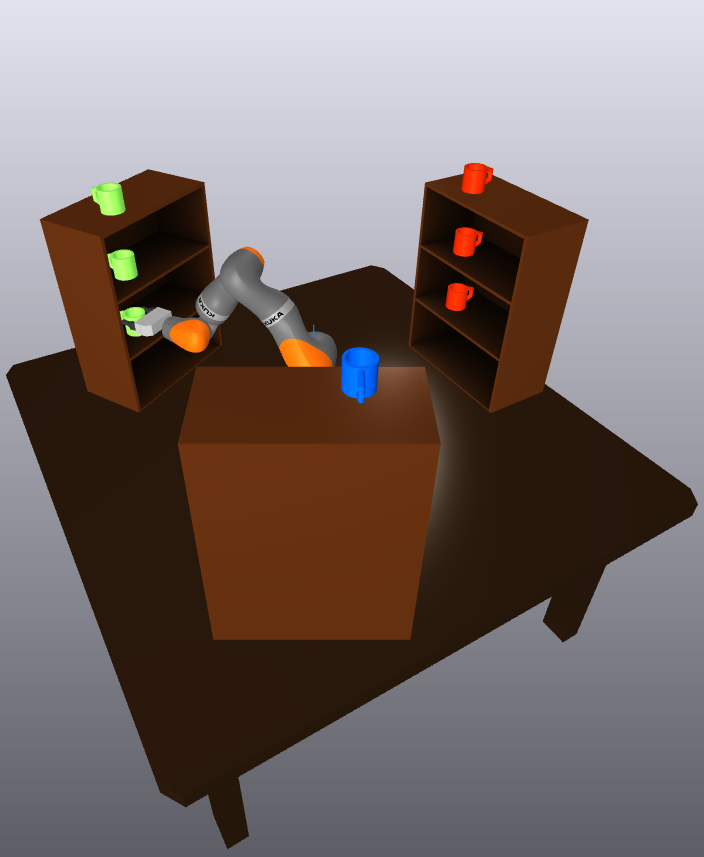}
	\end{subfigure}
	\caption{
		Initial (left) and final (right) states for the mug reorganization demo.
	}
	\label{fig:kuka_before_after}
\end{figure}

We also demonstrate that GGCS-Planner can be used to plan a series of motions using a KUKA iiwa robot arm.
The KUKA iiwa is a 7-DoF robot arm where each joint is a revolute joint with limits; in simulation, we remove the limits on the first joint, so the configuration space is $\mb{T}^{1}\times\mc{I}^{6}$.
We consider a scenario where the arm is mounted on a table, surrounded by three sets of shelves, with mugs arranged on the shelves.
The goal of the task is to sort the mugs onto different shelves, organized by color.
We specify the order of motions that are needed to achieve this goal and use GGCS-Planner to find the path from start to goal for each motion.

For this experiment, we used a set of 18 convex regions to achieve approximate coverage of the collision-free space.
These regions were adjusted as the mugs were moved about the environment and were used to plan the complete motion of the arm -- no heuristic motion or ``pre-grasp pose'' was needed to reach the grasp configuration.
Several configurations used to seed the region generation are shown in \Cref{fig:kuka_seeds}, and the initial and final states are shown in \Cref{fig:kuka_before_after}.
A video and an interactive recording of the plan are available at our \href{https://ggcs-anonymous-submission.github.io/}{results website}.

The experiment consisted of 14 motions, which were each planned individually, and we use the region refinement method from~\cite{petersen2023growing} to account for the current placement of the mugs.
This ensured that both the arm \emph{and the grasped mug} were collision-free for the entirety of each trajectory segment.
The robot takes full advantage of the base joint's lack of limits -- always choosing the shortest path and never needing to unwind any rotations.
For each segment, planning a trajectory took an average of $25.75$ seconds (with a range of $4.63$ to $50.30$ seconds).

\subsection{PR2 Bimanual Mobile Manipulator}
\label{sec:experiments:pr2}

In addition to its mobile base, the PR2 has two continuous revolute joints in each arm.
We have fixed the wrist rotation and gripper joints, so the configuration space is $\SE(2)\times\mb{T}^{2}\times\mc{I}^{10}\cong\mb{T}^{3}\times\mc{I}^{12}$.
We consider a scenario where the robot is driving around a square table that has an outward-pointing stack of three shelves on each side.
The robot must reach into the different shelves with both arms.
This represents a challenging motion planning scenario for all existing approaches due to the obstacle-rich environment and high dimensionality of the configuration space.

The performance of GGCS-Planner is largely driven by the choice of g-convex sets.
For each set of shelves, we generate IRIS regions for the robot to reach into the top, middle, and bottom shelves with both arms simultaneously.
We also generate two additional regions where the robot reaches into the middle shelf with one arm and the bottom shelf with the other while crossing its arms.
Finally, we manually add various intermediate regions to promote graph connectivity and cover more of C-Free.
In \Cref{fig:pr2-plan}, we show several robot configurations along a trajectory produced by GGCS-Planner.

For the planning scenarios considered with the PR2, we compare GGCS-Planner to existing approaches.
Trajectory lengths are listed in \Cref{tab:trajectory-lengths}, plan times are listed in \Cref{tab:plan-times}, and interactive recordings of all trajectories are available online at our \href{https://ggcs-anonymous-submission.github.io/}{results website}.
We compare our algorithm to kinematic trajectory optimization~\cite[\S7.2]{tedrake2022manipulation} (abbreviated as \emph{Drake-Trajopt}), utilizing the general nonlinear program solver SNOPT~\cite{snopt,snoptmanual}.
Drake-Trajopt struggles to figure out how to move the arms into or out of the shelf; we often have to add waypoints to force the robot to back out of the shelf by moving its base.

We also compare it to a sampling-based PRM planner.
To mitigate the curse of dimensionality and ensure connectivity between seed points, we initialize the PRM with the skeleton of the GGCS graph: for each pair of overlapping regions, we place a node in the Chebyshev center~\cite[p.~148]{boyd2004convex} of their intersection.
We then add 100,000 additional samples, drawn uniformly across C-Free (with rejection sampling).
This process takes $124.39$ seconds.
In comparison, it takes an average of $30.20$ seconds to generate an IRIS region (with a range of $8.56$ to $75.42$ seconds).
With parallelization, all of the IRIS regions were generated in only $156.63$ seconds.

The plans produced by the PRM are significantly longer than those from the GGCS-Planner, so we also examine using the output of the PRM planner as the initial guess for the trajectory optimization.
(In principle, this should help prevent Drake-Trajopt from getting stuck in local minima.)
When post-processing PRM plans with Drake-Trajopt, it sometimes produces shorter final trajectories than GGCS-Planner, at the expense of colliding slightly with the environment (an example is shown in \Cref{fig:trajopt-collision}).
This is likely due to the challenge of balancing the collision-free constraint with the minimum distance objective (and because collision-free constraints can only be applied at discrete points).

Finally, we compare our GGCS-Planner to two workarounds for applying ordinary GCS to non-Euclidean motion planning.
One could add artificial joint limits to prevent the wraparound, but placing the joint limits incorrectly could make the optimal path infeasible.
The planar arm experiment clearly demonstrates this problem; during the second trajectory in \Cref{fig:planar-arm}, the middle joint traverses more than 360\textdegree{} in the course of the plan.
Thus, the optimal trajectory is infeasible for \emph{every} possible choice of joint limits.

Another option is treating the angles as real numbers with no bound (and ignoring the fact that $0\ptxt{\textdegree}\equiv 360\ptxt{\textdegree}$).
But in this case, the correct joint angle modulo 360\textdegree{} must be chosen to get the optimal path.
Furthermore, many copies of each convex set must be made to account for each possible choice of angle modulo 360\textdegree{}, increasing the size of the optimization problem.

With both workarounds, a priori knowledge about the solution is required to guarantee that it is found, so in each comparison, we separately consider the best and worst cases.
We use the same seed points across GGCS and both GCS workarounds.

\newcommand{\clevernumspace}[1]{$\,#1\;$}
\newcommand{\faketwolines}{\vphantom{\parbox{1cm}{a\\ b}}}
\newcommand{\fakethreelines}{\vphantom{\parbox{1cm}{a\\ b\\ c}}}
\newcommand{\strikethroughverticalspace}{\vphantom{\sum}}

\begin{table*}
	\centering
	\setlength{\tabcolsep}{0.5em}
	\begin{tabular}{|>{\centering}p{1.5cm}|c|>{\centering}p{2cm}|c|>{\centering}p{3cm}|c|c|} \hline
		Experiment & GGCS-Planner & Drake-Trajopt & PRM & PRM + Drake-Trajopt & GCS-Planner (Joint Limits) & GCS-Planner (No Joint Limits)\\ \hline
		1T to 1B \faketwolines & 
			\href{https://ggcs-anonymous-submission.github.io/meshcat/1T-1B/ggcs.html}{\clevernumspace{1.829}} & 
			\href{https://ggcs-anonymous-submission.github.io/meshcat/1T-1B/trajopt.html}{\clevernumspace{1.803}} & 
			\href{https://ggcs-anonymous-submission.github.io/meshcat/1T-1B/prm.html}{\clevernumspace{4.359}} & 
			\href{https://ggcs-anonymous-submission.github.io/meshcat/1T-1B/hybrid.html}{\clevernumspace{1.808}} & 
			\href{https://ggcs-anonymous-submission.github.io/meshcat/1T-1B/naive_gcs_limits.html}{\clevernumspace{1.826}} & 
			\href{https://ggcs-anonymous-submission.github.io/meshcat/1T-1B/naive_gcs_no_limits.html}{\clevernumspace{1.839}}\\ \hline
		1CL to 1CR \faketwolines & 
			\href{https://ggcs-anonymous-submission.github.io/meshcat/1CL-1CR/ggcs.html}{\clevernumspace{2.255}} & 
			\href{https://ggcs-anonymous-submission.github.io/meshcat/1CL-1CR/trajopt.html}{\clevernumspace{2.204}} & 
			\href{https://ggcs-anonymous-submission.github.io/meshcat/1CL-1CR/prm.html}{\clevernumspace{9.219}} & 
			\href{https://ggcs-anonymous-submission.github.io/meshcat/1CL-1CR/hybrid.html}{\clevernumspace{2.182}} & 
			\href{https://ggcs-anonymous-submission.github.io/meshcat/1CL-1CR/naive_gcs_limits.html}{\clevernumspace{2.239}} & 
			\href{https://ggcs-anonymous-submission.github.io/meshcat/1CL-1CR/naive_gcs_no_limits.html}{\clevernumspace{2.247}}\\ \hline
		1M to 4M \fakethreelines & 
			\href{https://ggcs-anonymous-submission.github.io/meshcat/1M-4M/ggcs.html}{\clevernumspace{3.875}} & 
			\parbox{\linewidth}{\centering\href{https://ggcs-anonymous-submission.github.io/meshcat/1M-4M/trajopt.html}{\clevernumspace{\cancel{\strikethroughverticalspace 6.938}}}\\ $t=0.275,5.272$} & 
			\href{https://ggcs-anonymous-submission.github.io/meshcat/1M-4M/prm.html}{\clevernumspace{14.554}} & 
			\parbox{\linewidth}{\centering\href{https://ggcs-anonymous-submission.github.io/meshcat/1M-4M/hybrid.html}{\clevernumspace{\cancel{\strikethroughverticalspace 5.874}}}\\ $t=0.714,4.381$} & 
			\href{https://ggcs-anonymous-submission.github.io/meshcat/1M-4M/best_naive_gcs_limits.html}{\clevernumspace{6.482}} /
			\href{https://ggcs-anonymous-submission.github.io/meshcat/1M-4M/worst_naive_gcs_limits.html}{\clevernumspace{10.478}} & 
			\href{https://ggcs-anonymous-submission.github.io/meshcat/1M-4M/best_naive_gcs_no_limits.html}{\clevernumspace{3.990}} /
			\href{https://ggcs-anonymous-submission.github.io/meshcat/1M-4M/worst_naive_gcs_no_limits.html}{\clevernumspace{12.589}}\\ \hline
		1CL to 2CR \fakethreelines & 
			\href{https://ggcs-anonymous-submission.github.io/meshcat/1CL-2CR/ggcs.html}{\clevernumspace{4.473}} & 
			\parbox{\linewidth}{\centering\href{https://ggcs-anonymous-submission.github.io/meshcat/1CL-2CR/trajopt.html}{\clevernumspace{\cancel{\strikethroughverticalspace 5.409}}}\\ $t=2.155$} & 
			\href{https://ggcs-anonymous-submission.github.io/meshcat/1CL-2CR/prm.html}{\clevernumspace{12.110}} & 
			\parbox{\linewidth}{\centering\href{https://ggcs-anonymous-submission.github.io/meshcat/1CL-2CR/hybrid.html}{\clevernumspace{\cancel{\strikethroughverticalspace 4.108}}}\\ $t=0.49$} & 
			\href{https://ggcs-anonymous-submission.github.io/meshcat/1CL-2CR/best_naive_gcs_limits.html}{\clevernumspace{4.441}} /
			\href{https://ggcs-anonymous-submission.github.io/meshcat/1CL-2CR/worst_naive_gcs_limits.html}{\clevernumspace{13.815}} & 
			\href{https://ggcs-anonymous-submission.github.io/meshcat/1CL-2CR/best_naive_gcs_no_limits.html}{\clevernumspace{4.640}} /
			\href{https://ggcs-anonymous-submission.github.io/meshcat/1CL-2CR/worst_naive_gcs_no_limits.html}{\clevernumspace{13.233}}\\ \hline
		1CL to 3CR \fakethreelines & 
			\href{https://ggcs-anonymous-submission.github.io/meshcat/1CL-3CR/ggcs.html}{\clevernumspace{8.182}} & 
			\href{https://ggcs-anonymous-submission.github.io/meshcat/1CL-3CR/trajopt.html}{\clevernumspace{10.263}} & 
			\href{https://ggcs-anonymous-submission.github.io/meshcat/1CL-3CR/prm.html}{\clevernumspace{15.250}} & 
			\parbox{\linewidth}{\centering\href{https://ggcs-anonymous-submission.github.io/meshcat/1CL-3CR/hybrid.html}{\clevernumspace{\cancel{\strikethroughverticalspace 7.166}}}\\ $t=0.7,1.87,2.02,2.77$} & 
			\href{https://ggcs-anonymous-submission.github.io/meshcat/1CL-3CR/best_naive_gcs_limits.html}{\clevernumspace{7.820}} /
			\href{https://ggcs-anonymous-submission.github.io/meshcat/1CL-3CR/worst_naive_gcs_limits.html}{\clevernumspace{12.125}} & 
			\href{https://ggcs-anonymous-submission.github.io/meshcat/1CL-3CR/best_naive_gcs_no_limits.html}{\clevernumspace{8.501}} /
			\href{https://ggcs-anonymous-submission.github.io/meshcat/1CL-3CR/worst_naive_gcs_no_limits.html}{\clevernumspace{12.125}}\\ \hline
		1CL to 4CR \fakethreelines & 
			\href{https://ggcs-anonymous-submission.github.io/meshcat/1CL-4CR/ggcs.html}{\clevernumspace{4.382}} & 
			\href{https://ggcs-anonymous-submission.github.io/meshcat/1CL-4CR/trajopt.html}{\clevernumspace{7.583}} & 
			\href{https://ggcs-anonymous-submission.github.io/meshcat/1CL-4CR/prm.html}{\clevernumspace{17.459}} & 
			\parbox{\linewidth}{\centering\href{https://ggcs-anonymous-submission.github.io/meshcat/1CL-4CR/hybrid.html}{\clevernumspace{\cancel{\strikethroughverticalspace 6.088}}}\\ $t=0.27,0.555,4.39$} & 
			\href{https://ggcs-anonymous-submission.github.io/meshcat/1CL-4CR/best_naive_gcs_limits.html}{\clevernumspace{4.728}} /
			\href{https://ggcs-anonymous-submission.github.io/meshcat/1CL-4CR/worst_naive_gcs_limits.html}{\clevernumspace{9.961}} & 
			\href{https://ggcs-anonymous-submission.github.io/meshcat/1CL-4CR/best_naive_gcs_no_limits.html}{\clevernumspace{4.559}} /
			\href{https://ggcs-anonymous-submission.github.io/meshcat/1CL-4CR/worst_naive_gcs_no_limits.html}{\clevernumspace{12.418}}\\ \hline
		1T to 4B \fakethreelines & 
			\href{https://ggcs-anonymous-submission.github.io/meshcat/1T-4B/ggcs.html}{\clevernumspace{4.538}} & 
			\href{https://ggcs-anonymous-submission.github.io/meshcat/1T-4B/trajopt.html}{\clevernumspace{8.781}} & 
			\href{https://ggcs-anonymous-submission.github.io/meshcat/1T-4B/prm.html}{\clevernumspace{12.351}} & 
			\parbox{\linewidth}{\centering\href{https://ggcs-anonymous-submission.github.io/meshcat/1T-4B/hybrid.html}{\clevernumspace{\cancel{\strikethroughverticalspace 5.949}}}\\ $t=0.34,0.68$} & 
			\href{https://ggcs-anonymous-submission.github.io/meshcat/1T-4B/best_naive_gcs_limits.html}{\clevernumspace{5.320}} /
			\href{https://ggcs-anonymous-submission.github.io/meshcat/1T-4B/worst_naive_gcs_limits.html}{\clevernumspace{14.928}
			} & 
			\href{https://ggcs-anonymous-submission.github.io/meshcat/1T-4B/best_naive_gcs_no_limits.html}{\clevernumspace{5.473}} /
			\href{https://ggcs-anonymous-submission.github.io/meshcat/1T-4B/worst_naive_gcs_no_limits.html}{\clevernumspace{14.160}}\\ \hline
	\end{tabular}
	\caption{
		A comparison of trajectory lengths (in configuration space) for each PR2 experiment across different methods.
		Experiments are titled by the start and goal configurations.
		The configuration names indicate the shelf positions on the table (1 through 4), followed by the position of the grippers.
		T, M, B, CL, and CR stand for top, middle, bottom, cross left over right, and cross right over left (respectively).
		Table cells that are struck through indicate that the trajectory is not collision-free, and the time stamps below the trajectory length indicate when the collisions occurred.
		For both GCS-Planner workarounds, we include the best- and worst-case results (in general, achieving the best-case results requires a priori knowledge of the optimal plan).
		Interactive recordings of each trajectory are available online at our \href{https://ggcs-anonymous-submission.github.io/}{results website}.
		Each cell is linked to its corresponding recording.
	}
	\label{tab:trajectory-lengths}
\end{table*}

\begin{table}
	\centering
	\setlength{\tabcolsep}{0.25em}
	\begin{tabular}{|c|c|c|c|c|c|c|} \hline
		Experiment & GGCS-Planner & Drake-Trajopt & PRM & PRM + Drake-Trajopt\\ \hline
		1T to 1B & 25.51 & 12.69 & 0.49 & 11.61 \\ \hline
		1CL to 1CR & 39.42 & 15.23 & 0.49 & 16.11 \\ \hline
		1M to 4M & 46.61 & 2.26 & 0.53 & 25.51 \\ \hline
		1CL to 2CR & 62.87 & 9.74 & 0.54 & 21.48 \\ \hline
		1CL to 3CR & 58.60 & 7.82 & 0.52 & 27.30 \\ \hline
		1CL to 4CR & 66.15 & 4.32 & 0.54 & 40.10 \\ \hline
		1T to 4B & 29.89 & 10.92 & 0.54 & 15.36 \\ \hline
	\end{tabular}
	\caption{
		A comparison of online planning times (in seconds) for each PR2 experiment across different methods.
		(We omit the GCS workaround comparisons, as they are indistinguishable from the corresponding GGCS-Planner runtimes.)
		Experiment names match \Cref{tab:trajectory-lengths}.
	}
	\label{tab:plan-times}
\end{table}

\begin{figure}
	\centering
	\begin{overpic}[width=\linewidth]{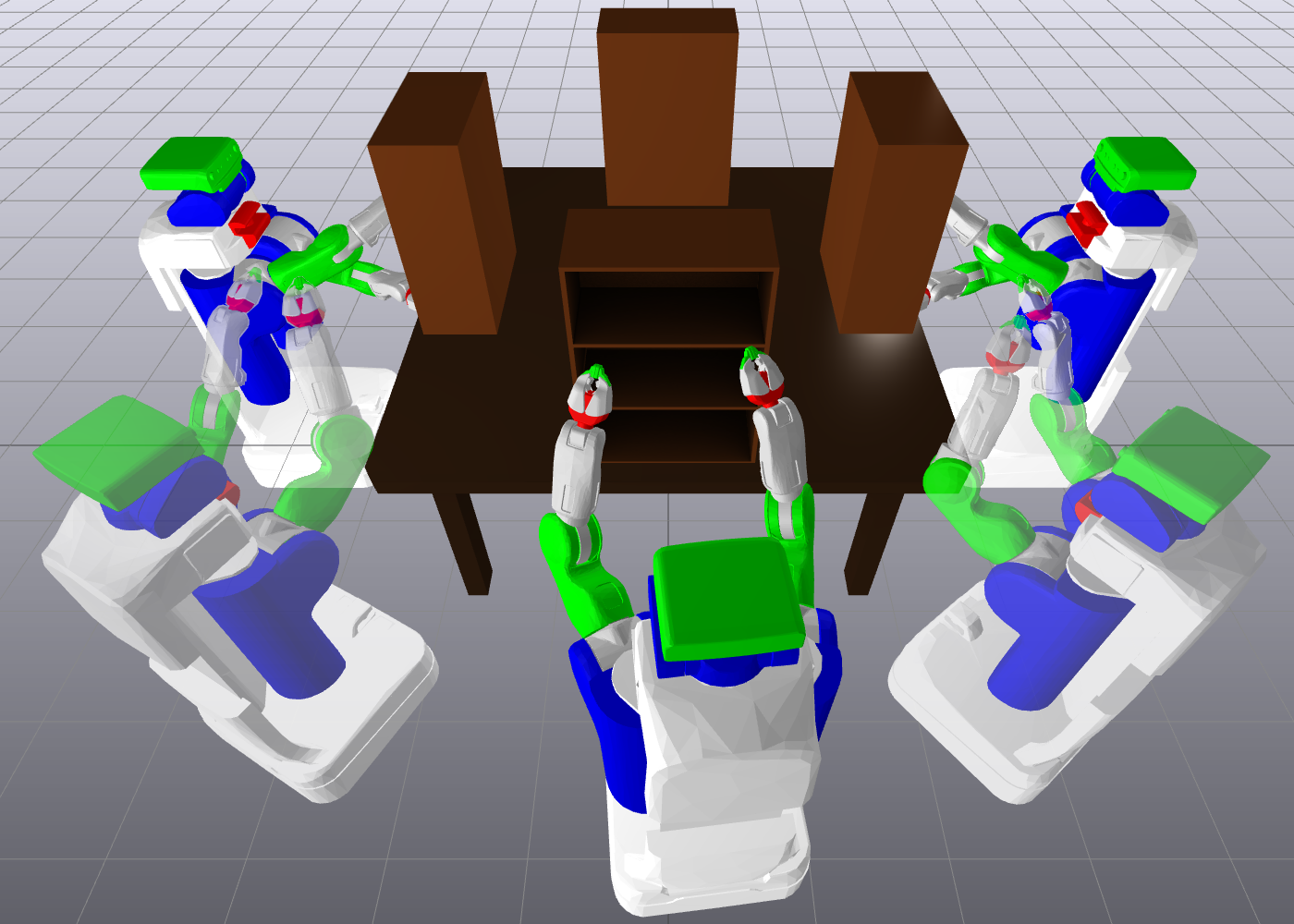}
		\put (85, 65) {(1)}
		\put (85, 10) {(2)}
		\put (40, 5) {(3)}
		\put (10,10) {(4)}
		\put (10,65) {(5)}
	\end{overpic}
	\caption{
		Individual poses along a trajectory produced by GGCS-Planner for the PR2 robot, labeled with their order in the plan.
	}
	\label{fig:pr2-plan}
\end{figure}
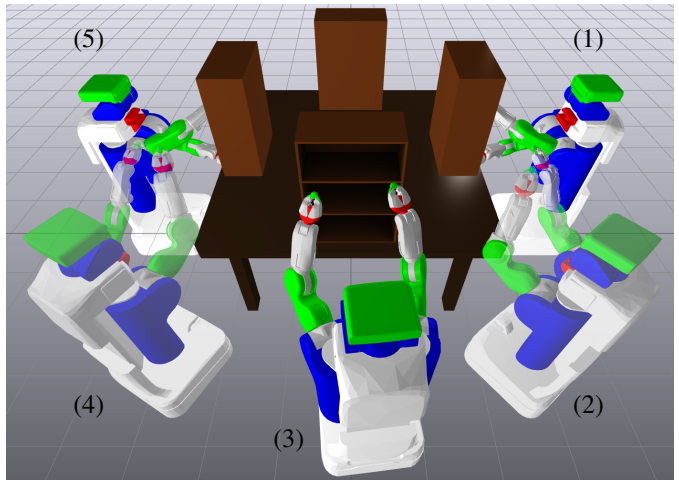

An interesting result is that the best case for the GCS workarounds is sometimes slightly better than GGCS.
This is because the sets are not bounded by the convexity radius, so they can grow larger (and cover more of C-Free) with the same seed points.
If the workarounds are restricted to using the same regions as GGCS, then, in the best case, their performance is indistinguishable.

\begin{figure}
	\centering
	\includegraphics[width=\linewidth]{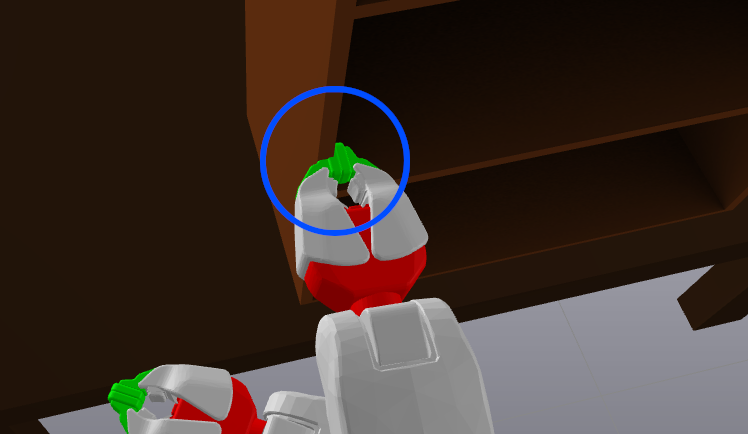}
	\caption{
		An example of the slight collisions typical of the trajectories produced by Drake-Trajopt. (The blue circle highlights the point of collision.)
	}
	\label{fig:trajopt-collision}
\end{figure}

\section{Discussion}
\label{sec:discussion}

In this paper, we have formulated the general problem of motion planning around obstacles on Riemannian manifolds as a shortest path problem in a graph of geodesically convex sets, and we have proved this formulation inherits the same guarantees as in the ordinary Euclidean case.
We describe how these theoretical developments inform simple and elegant modifications to the original GCS-Planner, in order to handle robots with mobile bases and continuous revolute joints.
This enables us to solve motion planning problems on such robotic platforms to global optimality and guarantee that the trajectory is collision-free at every point in time.
Approximate solving techniques still guarantee that trajectories are collision-free, and empirically, such trajectories are very close to optimal.

We have demonstrated that GGCS-Planner is a powerful tool for robot motion planning.
It is capable of producing plans for high degree-of-freedom systems operating in obstacle-dense configuration spaces, such as a PR2 bimanual mobile manipulator reaching into and out of shelves.
In future work, we intend to generalize the techniques outlined in this paper for use in bimanual manipulation (with implicit kinematic constraints) and planning through contact.
Although the planning and optimization frameworks used in GGCS-Planner are still in their infancy, they are already capable of producing high-quality results that are competitive with existing methods.
As further research and technical improvements are made, its performance will continue to improve.

\section*{Acknowledgments}
\label{sec:acknowledgements}

We would like to thank Tobia Marcucci and Seiji Shaw for their valuable suggestions throughout the course of this work.
Thank you to David von Wrangel for his assistance with the implementation of the PRM comparisons.
This work was supported by (in alphabetical order) Amazon.com, PO No. 2D-06310236 and the Frederick and Barabara Cronin Fellowship.

\bibliographystyle{unsrtnat}
\bibliography{ref}

\clearpage
\appendices
\renewcommand{\thesubsectiondis}{\Alph{section}.\arabic{subsection}}
\renewcommand{\thesubsection}{\Alph{section}.\arabic{subsection}}

{
\crefalias{section}{appendix}
\crefalias{subsection}{appendix}

\section{Proofs}
\label{appx:proofs}

\subsection{Proof of \texorpdfstring{\Cref{thm:burago}}{Theorem 1}}
\label{proof:burago}

\begin{lemma}
	For any $p,q\in\closure{\mc{M}}$, there is a piecewise-smooth path connecting $p$ and $q$.
	\label{lemma:piecewise_smooth_path}
\end{lemma}
\begin{proof}
	Because $\mc{M}$ is path connected, there is a continuous curve $\gamma:[a,b]\to\closure{\mc{M}}$ joining them.
	Let $(U_1,\psi_1),\ldots,(U_n,\psi_n)$ be a series of charts of $\mc{Q}$ that cover the image of $\gamma$, with $p\in U_1$, $q\in U_n$, and $U_i\cap U_{i+1}\cap\closure{\mc{M}}\ne\emptyset$ for each $i$.
	(Such a finite covering exists because the image of $\gamma$ is compact.)
	Let $t_0,\ldots,t_n\in[a,b]$ such that $t_0=a$, $t_n=b$, and for each $i=1,\ldots,n-1$, $\gamma(t_i)\in U_i\cap U_{i+1}$.
	For each $i=1,\ldots,n-1$, let $\tilde\gamma_i$ be a smooth curve joining $\psi_i(\gamma(t_i))$ to $\psi_{i+1}(\gamma(t_{i+1}))$ that is contained within $\psi_i(U_i\cap\closure{\mc{M}})$.
	Let $\tilde\gamma_0$ join $\psi_1(\gamma(t_0))$ to $\psi_1(\gamma(t_1))$ and be contained within $\psi_1(U_1\cap\closure{\mc{M}})$, and let $\tilde\gamma_n$ join $\psi_n(\gamma(t_{n-1}))$ to $\psi_n(q)$ and be contained within $\psi_n(U_n\cap\closure{\mc{M}})$.
	Then by lifting each of these curves to $\closure{\mc{M}}$, and concatenating them, we obtain a piecewise-smooth curve connecting $p$ and $q$.
\end{proof}

\hfill

\begin{proof}[Proof of \Cref{thm:burago}]
	The proof follows by verifying that $\closure{\mc M}$ is a complete, locally compact length space, so that we can apply Theorem 2.5.23 of \citet[p.~50]{burago2022course}.
	A length space is a metric space in which the distance between any two points is given by the infimum of the arc lengths of all paths connecting those two points.
	A length space is complete if the distance between any two points is finite.
	Thus, $\closure{\mc{M}}$ inherits a length structure from $\mc{Q}$ (with the restriction to curves that are entirely contained in $\closure{\mc{M}}$).
	All topological manifolds are locally compact.
	To check that $\closure{\mc{M}}$ is complete, let $p,q\in\closure{\mc{M}}$.
	By \Cref{lemma:piecewise_smooth_path}, there is a piecewise-smooth curve connecting $p$ and $q$, so the set of arc lengths of curves connecting $p$ and $q$ is nonempty.
	It is also bounded below, so its infimum is finite, and thus $\dist(p,q)$ exists.
	We conclude that $\closure{\mc{M}}$ is a complete, locally compact length space.
\end{proof}

\subsection{Proof of \texorpdfstring{\Cref{thm:convex_product}}{Theorem 4}}
\label{proof:convex_product}

\begin{lemma}
	Let $(\mc{Y}_1,\psi_1)$ and $(\mc{Y}_2,\psi_2)$ be coordinate charts of $\mc{M}$, with $\psi_1$ and $\psi_2$ local isometries, and $\mc{Y}_1$ and $\mc{Y}_2$ g-convex.
	Then there is a Euclidean isometry $\xi$ such that $\forall p\in\mc{Y}_1\cap\mc{Y}_2$, $\psi_1(p)=(\xi\of\psi_2)(p)$.
	\label{lemma:off_by_an_isometry}
\end{lemma}
\begin{proof}
	$\mc{Y}_1\cap\mc{Y}_2$ is g-convex, and hence connected.
	$\psi_1\of\psi_2\inv$ is a local isometry between two connected open subsets of Euclidean space (with appropriate restriction of domain and range), so $(\psi_1\of\psi_2\inv)_{*,p}$ is an orthogonal matrix for any $p$.
	Thus, we can apply Theorem 1.8-1 of \citet[p.~44]{ciarlet1988elasticity}.
\end{proof}

\begin{lemma}
	Consider $\mc{Y}\subseteq\mc{Z}\subseteq\mc{M}$, where $\mc{Z}$ is g-convex, and we have a coordinate chart $(\mc{Z},\psi)$ such that $\psi$ is a local isometry.
	If $\psi(\mc{Y})$ is convex, then $\mc{Y}$ is g-convex.
	\label{lemma:convex_in_coordinates}
\end{lemma}
\begin{proof}
	Fix $p,q\in\mc{Y}$.
	Then there is a unique minimizing geodesic $\gamma$ connecting $p$ to $q$, and $\gamma$ is contained in $\mc{Z}$.
	Because $\psi$ is a local isometry, it maps $\gamma$ to a line segment in $\psi(\mc{Z})$.
	$\psi(p),\psi(q)\in\psi(\mc{Y})$, so by convexity of $\psi(\mc{Y})$, $\psi\of\gamma$ is contained in $\psi(\mc{Y})$.
	Thus, $\gamma$ is contained in $\mc{Y}$, so $\mc{Y}$ is g-convex.
\end{proof}

\hfill

\begin{proof}[Proof of \Cref{thm:convex_product}]
	For each $i\in[m]$, we can construct a ball $B_{s_i}(c_i)\supseteq\proj_{\mc{Q}_i}(\mc{Y})$, with $s_i<r_i$.
	Define $\mc{Z}=\prod_{i\in[m]}B_{s_i}(c_i)$, a g-convex set.
	Consider the Riemannian normal coordinates of $\mc{Q}$ at $(c_1,\ldots,c_m)$.
	This coordinate system, restricted to $\mc{Z}$, induces a coordinate chart $\varphi$.
	Because $\mc{Q}$ is flat, $\varphi$ is a local isometry, so by \Cref{lemma:off_by_an_isometry} there is a Euclidean isometry $\xi$ such that $\varphi(\mc{Y})=\xi(\psi(\mc{Y}))$, so $\varphi(\mc{Y})$ is convex.
	Thus, by \Cref{lemma:convex_in_coordinates}, $\mc{Y}$ is g-convex.
\end{proof}

\begin{figure}
	\centering
	\includegraphics[width=0.6\linewidth]{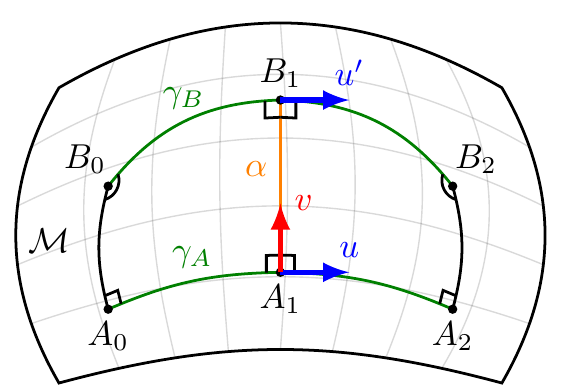}
	\caption{
		The construction of two Levi-Civita parallelogramoids used in the proof of Theorem~\ref{thm:positive_nonconvexity}.
		$\dist(A_0,B_0)<\dist(A_1,B_1)$ and $\dist(A_2,B_2)<\dist(A_1,B_1)$, which demonstrates the nonconvexity of $\dist$ around $A_1$.
		$\angle A_0 B_0 B_1$ and $\angle A_2 B_2 B_1$ are obtuse.
	}
	\label{fig:positive-curvature}
\end{figure}

\subsection{Proof of \texorpdfstring{\Cref{thm:positive_nonconvexity}}{Theorem 5}}
\label{proof:positive_nonconvexity}

\begin{proof}[Proof of \Cref{thm:positive_nonconvexity}]
	Fix a neighborhood $\mc{U}$ of $A_1$.
	Since the sectional curvature is invariant with respect to a change of basis, suppose without loss of generality that $\norm{u}=\norm{v}=1$ and $\iprod{u,v}=0$.
	To prove this result, we will construct a geodesic $\gamma$ on $\mc{U}^{2}$ such that $\dist\of\gamma$ achieves smaller values on its endpoints than at its center.
	This relies on the properties of specially constructed Levi-Civita parallelogramoids on $\mc{U}$.

	Let $\epsilon>0$ be small, such that $\exp_p(\overline{B_{2\epsilon}(0)})\subseteq\mc{U}$.
	Let $\alpha:t\mapsto\exp_p(tv)$, let $B_1=\alpha(\epsilon)$, and let $u'=\Gamma(\alpha)_0^\epsilon(u)$ be $u$ parallel transported from $A_1$ to $B_1$.
	Let $\gamma_A:t\mapsto\exp_{A_1}(tu)$ and $\gamma_B:t\mapsto\exp_{B_1}(tu')$, with domain $[-\epsilon,\epsilon]$.
	Then $\gamma=(\gamma_A,\gamma_B)$ is a geodesic of $\mc{U}^2$.
	Define $A_0=\gamma_A(-\epsilon)$, $B_0=\gamma_B(-\epsilon)$, $A_2=\gamma_A(\epsilon)$, and $B_2=\gamma_B(\epsilon)$.
	This construction is visualized in \Cref{fig:positive-curvature}.
	We want to show that $\dist(A_0,B_0)<\dist(A_1,B_1)$ and $\dist(A_2,B_2)<\dist(A_1,B_1)$.

	The points $A_1$, $B_1$, $A_2$, and $B_2$ describe a Levi-Civita parallelogramoid, with base $A_1B_1$ and suprabase $A_2B_2$.
	Thus, we can relate the length of the base and suprabase via the formula of~\cite[p.~244]{cartan1983geometry}
	\[
		\dist(A_2,B_2)^2=\dist(A_1,B_1)^2+\tfrac{8}{3}\iprod{\mc{R}(\epsilon u,\epsilon v)\epsilon u,\epsilon v}+O(\epsilon^5)
	\]
	Because $\norm{u}=\norm{v}=1$ and $\iprod{u,v}=0$,
	\[
		\iprod{\mc{R}(\epsilon u,\epsilon v)\epsilon u,\epsilon v}=-\epsilon^4\iprod{\mc{R}(u,v)v,u}=-\epsilon^4K(u,v)<0
	\]
	So for $\epsilon$ is decreased towards $0$, the fifth and higher order terms vanish, and $\dist(A_2,B_2)<\dist(A_1,B_1)$.
	A similar calculation shows that $\dist(A_0,B_0)<\dist(A_1,B_1)$.
	Thus, $d\of\gamma$ has a local minimum, so we conclude that $d$ is nonconvex on $\mc{U}$.
\end{proof}

\section{Experiment Implementation Details}
\label{appx:experiment_implementation}

In this appendix, we present further details about the setup of our experiments and demonstrations.

\subsection{Planar Arm}
\label{subappx:planar_arm_details}

The trajectories shown in \Cref{sec:experiments:planar_arm} were generated with a GGCS that had 19 sets.
We generated IRIS regions for the start and goal configurations, and hand-picked several seed points along the narrow gap between the two lower obstacles to help ensure connectivity between the start and goal.
We then generated the remaining IRIS regions with random seed points (chosen uniformly from C-Free with rejection sampling).

The GGCS-Planner results shown used the sum of the path length and path duration as the objective.
We used the relax-and-round approximation strategy to produce the trajectories shown in the paper.
The first trajectory had a path length of 7.749, and the second had a path length of 8.448.
When solving the integer program with branch-and-bound, the first trajectory had a path length of 7.274, and the second had a path length of 8.008.
(Note that the optimal solution for the latter trajectory still had the middle joint of the arm traverse more than 360\textdegree.)

\subsection{KUKA iiwa Arm}
\label{subappx:kuka_iiwa_details}

The motions shown in \Cref{sec:experiments:kuka_iiwa} used regions generated from 18 seed points.
The seeds consisted of one seed for each middle and top shelf (the bottom shelves are excluded because they are kinematically unreachable), one seed above each shelf, one seed directly between each shelf, and two seeds per shelf to aid moving between the top and middle shelves.
Regions were generated for each seed with both an empty hand and a mug in the hand to aid both types of trajectory planning.
Regions were post-processed to remove redundant hyperplanes with the \href{https://drake.mit.edu/doxygen_cxx/classdrake_1_1geometry_1_1optimization_1_1_h_polyhedron.html}{\texttt{ReduceInequalities}} algorithm from Drake.

The GGCS-Planner minimized both time and path length of the trajectory while ensuring continuity of the path, velocity, and acceleration.
For velocity limits, the real velocity limits of the KUKA iiwa hardware were used.
Trajectories were computed using the relax-and-round approximation strategy.

\subsection{PR2 Bimanual Mobile Manipulator}
\label{subappx:pr2_details}

To model the PR2 robot, we use the URDF file and object meshes included with Drake.
For each link, we take the convex hull of the mesh and use that as the collision geometry.
(Collisions annotated in \Cref{tab:trajectory-lengths} are determined based on the true collision geometry, not the convex hulls.)
The plans we produce take into account the robot's base joint, torso lift joint, and all arm joints (up to the final wrist rotation joint and gripper joints).
All other joints are fixed.

For the experiments demonstrated in \Cref{sec:experiments:pr2}, we first constructed IRIS regions for each of the possible goals: reaching into each of the three shelves in a set with both arms, crossing right-over-left on the middle and bottom shelves, and crossing left-over-right.
(See \Cref{fig:teaser} for a visualization of these cross-over poses.)
We then hand select a few intermediate seed points; the regions around these points are used to promote connectivity among the various shelf-reaching regions.
We construct these regions for each set of shelves, except for the experiments where the start and goal are within the same set of shelves.

We take several actions to improve the efficiency of GGCS-Planner.
To reduce the number of constraints needed, we simplify the IRIS regions by removing redundant halfspaces from their polyhedral representation, using the \href{https://drake.mit.edu/doxygen_cxx/classdrake_1_1geometry_1_1optimization_1_1_h_polyhedron.html}{\texttt{ReduceInequalities}} algorithm in Drake.
We also only include shelf-reaching regions if they are the start or goal of the plan.
This greatly reduces the size of the optimization problem, promoting faster solve times.
Empirically, it also leads to a shorter trajectory, likely due to a tightening of the convex relaxation.
For GGCS-Planner, we use the same objective as the planar arm experiments (the sum of the trajectory length and duration), and we use the relax-and-round strategy.

For the comparison to kinematic trajectory optimization (Drake-Trajopt), we use the same objective as GGCS-Planner: the sum of the trajectory duration and length.
However, the trajectories are parametrized as B-splines instead of linear segments (or \bz{} curves if the extensions in \Cref{subsec:further_motion_planning} are utilized).
The \href{https://drake.mit.edu/doxygen_cxx/classdrake_1_1systems_1_1trajectory__optimization_1_1_kinematic_trajectory_optimization.html}{\texttt{KinematicTrajectoryOptimization}} can automatically construct the nonlinear optimization problem for a given scenario, which we then solve with SNOPT.
We first solve the problem without collision-free constraints.
The output of this initial problem is used as the initial guess for the full problem (i.e., including collision-free constraints).
The collision-free constraint is encoded with the \href{https://drake.mit.edu/doxygen_cxx/classdrake_1_1multibody_1_1_minimum_distance_constraint.html}{\texttt{MinimumDistanceConstraint}} class.
We set a minimum distance of 1mm and begin applying a penalty at 1cm.
This constraint is applied to 50 points along the trajectory.
(Such a constraint can only be evaluated pointwise.)
For motion planning tasks where the robot had to move between shelves, Drake-Trajopt was unable to produce a collision-free trajectory.
Thus, we added waypoints near the beginning and end of the trajectory, in which the robot was in the same configuration as the start and goal (respectively), but the base was moved away from the shelf.
This was only sometimes effective at finding collision-free trajectories.

As in \cite{marcucci2022motion}, we use the PRM planner from the Common Robotics Utilities library~\cite{common_robotics_utilities}, with the modifications described in \Cref{sec:experiments:pr2}.
Given a piecewise-linear trajectory from the PRM, we construct a B-spline that passes through the nodes on this trajectory for use as an initial guess for Drake-Trajopt.
In this case, when solving the optimization problem, we begin applying a distance penalty at 1m and perform collision checking at 100 points along the trajectory.
}

\end{document}